%% file: tech_main.tex
\documentclass[oneside,a4paper,onecolumn,11pt]{article}

\usepackage{fullpage}
\usepackage[left=2cm,top=2cm,bottom=2cm,right=2cm,includehead,nomarginpar,headheight=16pt]{geometry}
\usepackage[ruled, linesnumbered, vlined, commentsnumbered]{algorithm2e}
\usepackage{graphicx,subfig} 
\usepackage{amsfonts,amssymb,amsmath,amsthm,amsopn,mathtools}	
\usepackage{booktabs,diagbox,colortbl,multirow,tabularx,threeparttable,hhline}
\usepackage[listings,skins,breakable]{tcolorbox}

\usepackage{fancyhdr,fancyheadings,nopageno,lastpage} 
\usepackage{url}
\usepackage{enumerate}
\usepackage[shortlabels]{enumitem}
\usepackage{csquotes}
\usepackage{authblk}
\usepackage{footnote}
\usepackage{hyperref}
\usepackage{prettyref}
\usepackage{cite}
\usepackage{setspace}
\usepackage{color}
\usepackage{xcolor}  
\usepackage{geometry}
\usepackage{academicons}
\usepackage{fontawesome5}
\usepackage{pifont,ifsym,marvosym,manfnt} 

\usepackage{tikz}
\usepackage{pgfplots}
\usetikzlibrary{positioning,shapes,shadows,arrows,calc}
\tikzstyle{component}=[rectangle, draw=black, rounded corners, fill=blue!40, drop shadow, text centered, anchor=north, text=white, minimum height=1cm]
\tikzstyle{arrow}=[->, thick]

\pgfplotsset{compat=1.12}
\usetikzlibrary{intersections}
\usetikzlibrary{pgfplots.statistics}
\usepgfplotslibrary{fillbetween}

\fancypagestyle{plain}{%
    \fancyfoot[C]{}
    \fancyfoot[R]{\faLeanpub\ \, \thepage\ / \pageref{LastPage}}
    \fancyfoot[L]{\faBraille\ \textsc{COLALab Report} \faSlackHash\ \footnotesize\textifsym{2024001}}
}

\hypersetup{
    colorlinks=true,
    linkcolor=ultramarine,
    filecolor=magenta,
    urlcolor=ultramarine,
    pdftitle={Overleaf Example},
    pdfpagemode=FullScreen,
}

\definecolor{red(munsell)}{rgb}{0.95, 0.0, 0.24}
\definecolor{navyblue}{RGB}{0, 0, 128}
\definecolor{myblue}{RGB}{34,31,217}
\definecolor{mycyan}{gray}{.7}
\definecolor{Gray}{gray}{0.9}
\definecolor{usccardinal}{rgb}{0.6, 0.0, 0.0}
\definecolor{ultramarine}{RGB}{0,32,96}
\definecolor{amber}{rgb}{1.0, 0.49, 0.0}

\newtheorem{remark}{Remark}
\newtheorem{theorem}{Theorem}

\newtheorem{definition}{Definition}
\newtheorem{lemma}{Lemma}

\newtcolorbox{quotebox}{colback=gray!10,boxrule=0.4pt,colframe=black,fonttitle=\bfseries,top=1pt,bottom=1pt}

\usepackage[ruled, linesnumbered, vlined, commentsnumbered]{algorithm2e}
\usepackage{prettyref}
\usepackage{amsfonts,amssymb,amsmath,amsthm,amsopn,mathrsfs,mathtools,wasysym}	
\usepackage{bm}
\usepackage{multirow}

\def\our{\texttt{CBOB}}

\SetCommentSty{mycommfont}

\newcommand{\pref}{\prettyref}

\newrefformat{fig}{Fig.~\ref{#1}}
\newrefformat{tab}{Table~\ref{#1}}
\newrefformat{sec}{Section~\ref{#1}}
\newrefformat{alg}{Algorithm~\ref{#1}}
\newrefformat{property}{Property~\ref{#1}}
\newrefformat{theorem}{Theorem~\ref{#1}}
\newrefformat{definition}{Definition~\ref{#1}}
\newrefformat{corollary}{Corollary~\ref{#1}}
\newrefformat{lemma}{Lemma~\ref{#1}}
\newrefformat{conj}{Conjecture~\ref{#1}}
\newrefformat{def}{Definition~\ref{#1}}
\newrefformat{eq}{equation~(\ref{#1})}
\newrefformat{app}{Appendix~\ref{#1}}

\usepackage{lscape}

\newenvironment{code-example}
{
\vspace{0.15cm}
\noindent\begin{minipage}{\linewidth}
\begin{center}
\arrayrulecolor{black}
\color{black}
\begin{tabular}{|p{0.95\linewidth}|}
\hline%
\rowcolor{pink!20}%
}
{
\\\hline
\end{tabular}
\end{center}
\end{minipage}
\vspace{-0.2cm}
}

\begin{document}

\title{\vspace{-1ex}\LARGE\textbf{Constrained Bayesian Optimization under Partial Observations: Balanced Improvements and Provable Convergence}~\footnote{This paper is accepted by AAAI 2024. Copyright © 2024, Association for the Advancement of Artificial Intelligence (www.aaai.org). All rights reserved.}}

\author[1]{\normalsize Shengbo Wang}
\author[2]{\normalsize Ke Li}
\affil[1]{\normalsize School of Computer Science and Engineering, University of Electronic Science and Technology of China, Chengdu 611731, China}
\affil[2]{\normalsize Department of Computer Science, University of Exeter, EX4 4RN, Exeter, UK}
\affil[\Faxmachine\ ]{\normalsize \texttt{shnbo.wang@foxmail.com}, \texttt{k.li@exeter.ac.uk}}

\date{}
\maketitle

\vspace{-3ex}
{\normalsize\textbf{Abstract: } }The partially observable constrained optimization problems (POCOPs) impede data-driven optimization techniques since an infeasible solution of POCOPs can provide little information about the objective as well as the constraints. We endeavor to design an efficient and provable method for expensive POCOPs under the framework of constrained Bayesian optimization. Our method consists of two key components. Firstly, we present an improved design of the acquisition functions that introduce balanced exploration during optimization. We rigorously study the convergence properties of this design to demonstrate its effectiveness. Secondly, we propose Gaussian processes embedding different likelihoods as the surrogate model for partially observable constraints. This model leads to a more accurate representation of the feasible regions compared to traditional classification-based models. Our proposed method is empirically studied on both synthetic and real-world problems. The results demonstrate the competitiveness of our method for solving POCOPs.


\input{introduction}

\input{related_work}

\input{preliminary}

\input{proposal}

\input{settings}

\input{experiments}

\input{conclusion}

\section*{Acknowledgment}
This work was supported in part by the UKRI Future Leaders Fellowship under Grant MR/S017062/1 and MR/X011135/1; in part by NSFC under Grant 62376056 and 62076056; in part by the Royal Society under Grant IES/R2/212077; in part by the EPSRC under Grant 2404317; in part by the Kan Tong Po Fellowship (KTP\textbackslash R1\textbackslash 231017); and in part by the Amazon Research Award and Alan Turing Fellowship.

\bibliographystyle{IEEEtran}
\bibliography{IEEEabrv,your_bib}

\newpage
\input{tech_appendix}

\end{document}

%% file: introduction.tex
\section{Introduction}
\label{sec:introduction}

The black-box constrained optimization problem considered in this paper is formulated as:
\begin{equation}
        \mathrm{minimize} \;  f(\mathbf{x}) \quad 
        \mathrm{subject\ to} \quad \vec{g}(\mathbf{x})\leq 0, 
    \label{eq:cop}
\end{equation}
where $\mathbf{x}=(x_1,\cdots,x_n)^\top\in\Omega$ denotes the decision variable, $\Omega=[x_i^\mathrm{L},x_i^\mathrm{U}]_{i=1}^n\subset\mathbb{R}^n$ denotes the search space, $x_i^\mathrm{L}$ and $x_i^\mathrm{U}$ are the lower and upper bounds of $x_i$ respectively. The objective function $f(\mathbf{x})$ and $m$ constraint functions $\vec{g}(\mathbf{x})=(g_1(\mathbf{x}),\cdots,g_m(\mathbf{x}))^\top$ are: $i$) \textit{analytically unknown}, i.e., we do not have access to $f$ and $\vec{g}$ directly, but to $\mathbf{x}$ to be determined and their observations $f(\mathbf{x})$ and $\vec{g}(\mathbf{x})$ instead; $ii$) \textit{computationally expensive}; and $iii$) \textit{partially observable}, i.e., the values of $f$ and $\vec{g}$ are not observable/measurable when $\mathbf{x}$ is infeasible. We denote the unknown feasible space by $\chi=\left\{\mathbf{x}\in\Omega|\vec{g}(\mathbf{x})\leq 0\right\}$. Partially observable constrained optimization problems (POCOPs) are not uncommon in real-life applications. For example, a robot control task will be suspended when a collision or excessive instantaneous power consumption is detected, where the feedback is merely a failure message rather than any reward~\cite{MarcoBKHRT21}. An AutoML task will be terminated without outputting the performance of a hyperparameter configuration but an error log when there is a memory overflow or computation timeout~\cite{PerroneSJAS19}.

Bayesian optimization (BO) is recognized as an effective query-efficient framework for black-box optimization \cite{garnett_bayesoptbook_2023}. Although there have been dedicated efforts on constraint handling in the context of BO, a.k.a. constrained Bayesian optimization (CBO), most of them are however expected to work with complete observations. 
Considering missing observations, existing CBO methods may become inefficient due to the following two issues.
\begin{itemize}
    \item First, existing CBO methods risk overly exploiting known feasible regions in POCOPs. In particular, the update of an acquisition function such as expected improvement (EI)~\cite{JonesSW98} stagnates when objective values are unobservable outside $\chi$. This stagnation may cause cluttered observations in local feasible regions, resulting in an overconfidence effect in both surrogate modeling and candidate acquisition. Consequently, the efficiency of BO diminishes.

    \item Second, POCOPs generate mixed data from both feasible and infeasible solutions, which complicate sufficient exploitation using conventional surrogate models. When a probabilistic classifier like Gaussian process classifier (GPC) is employed to distinguish between feasible and infeasible solutions, the available observations lose their utility in refining the GPC model. Brockman et al. \cite{BrockmanCPSSTZ16} leveraged value observations by using regression models and artificially injecting values into infeasible solutions. However, this approach potentially introduces erroneous priors, thereby compromising optimization efficacy. 
\end{itemize}

Bearing these considerations in mind, this paper proposes a novel CBO framework for POCOPs. Our main contributions are outlined as follows.
\begin{itemize}
    \item To address the risk of overly exploiting evaluated local feasible regions, we propose a novel acquisition function framework. It enhances EI with a balanced constraint handling technique, encapsulated in a general exploration function to enable global search during optimization. We theoretically analyze the convergence properties of this design, and further develop an instance of the exploration function, effectively enhancing the efficiency of our method for solving POCOPs.

    \item To fully leverage the mixed-observation characteristic of POCOPs, we propose heterogeneous-likelihood Gaussian processes (HLGPs), providing a promising representation of unknown constraints compared to classifier-based models. Further, we employ expectation propagation to manage the non-Gaussian inference, yielding an efficient Gaussian approximation of HLGPs.

    \item To demonstrate the efficacy of our proposed method, we perform a series of experiments on diverse benchmark problems. These include synthetic problems and real-world applications in reinforcement learning-based control design and machine learning hyperparameter optimization. Experimental results show the competitiveness and better efficiency of our method compared to selected state-of-the-art CBO methods.
\end{itemize}

%% file: related_work.tex
\section{Related Works}
\label{sec:related}

\subsection{Constrained Bayesian Optimization} 

Classic CBO methods usually reshape an unconstrained acquisition function by incorporating feasibility considerations. A prominent approach, the EI with constraints (EIC), first introduced by Schonlau et al. \cite{SchonlauWJ98}, has been extensively employed to locate feasible solutions with high probability~\cite{GardnerKZWC14,GelbartSA14}. To further enhance EIC's capability, Letham et al. \cite{LethamKOB17} leveraged a quasi-Monte Carlo approximation regarding observation noises. Furthermore, various strategies such as integration~\cite{GelbartSA14}, rollout~\cite{LamW17}, and a two-step lookahead algorithm~\cite{ZhangZF21}, have been proposed to achieve better optimization efficiency. These approaches increase the exploration of unknown regions, albeit at the cost of computational complexity, hindering their scalability for high-dimensional problems.

From the perspective of uncertainty reduction, predictive entropy search (PESC), which selects feasible candidate solutions directly from the search space, offers attractive heuristics to constrained optimization problems~\cite{LobatoGHAG15a}. However, the intractable nature of quadrature calculations during sampling in PESC has been a challenge. To address this, Takeno et al. \cite{TakenoTSK22} proposed the min-value entropy search, enabling the sampling process to operate more effectively within the objective space. This concept was subsequently extended to accommodate binary observations~\cite{PerroneSJAS19} and multi-objective scenarios~\cite{BelakariaDD19}. Besides, the fusion of EI and entropy search showed promise for enhancing exploration~\cite{LindbergL15}.

To harness the structure inherent in~\pref{eq:cop}, researchers such as Gramacy et al. \cite{GramacyGDLRWW16} and Picheny et al. \cite{PichenyGWD16} proposed the utilization of a Lagrangian method with slack variables, providing the capability to deal with equality constraints. Regarding problems marked by unknown constraints, Ariafar et al. \cite{AriafarCBD19} integrated BO with the alternating direction method of multipliers, thereby enabling the exploration of solutions even in the absence of feasible ones. Meanwhile, Eriksson and Poloczek \cite{ErikssonP21} proposed the use of Thompson sampling combined with a trust region approach towards enhancement of the scalability, meanwhile incorporating a thoughtful design aimed at maintaining computational efficiency.

\subsection{Surrogate Models for Constraints} 

The aforementioned CBO methods frequently employ Gaussian Process regression (GPR) and GPC to construct surrogate models for unknown constraints. Notably, classifiers such as GPC and support vector machine (SVM) have been found effective in sequential updates when the real value of an infringed constraint remains unobservable~\cite{LindbergL15,PerroneSJAS19,AriafarCBD19,BachocHP20,Candelieri21}. Yet, in partially observable scenarios, these classifiers exhibit limitations in utilizing available real-value observations, resulting in a dip in modeling performance. In response to this challenge, Marco et al. \cite{MarcoBKHRT21} enhanced the construction of GPR models by introducing a switched likelihood combined with mixed data observations. As an alternative solution, Pourmohamad and Lee \cite{PourmohamadL16} and Zhang et al. \cite{ZhangDL19} proposed multivariate GPs (MVGP) with joint distributions of hybrid input to handle mixed observations.

\subsection{Exploration for Unknown Feasible Regions}

In light of the black-box nature of the problem at hand, discerning the feasibility of a solution becomes a significant concern. In this context, Parr et al. \cite{ParrKFH12} interpreted the delicate balance between enhancing the probability of feasibility (POF) and optimizing the objective as a multi-objective optimization issue. Focusing on optimization, Picheny \cite{Picheny14} developed a step-wise uncertainty reduction method, which capitalizes on the volume of feasible regions beneath the most promising solution observed up to that point, despite the method's considerable computational complexity. Furthermore, the EIC was adapted in~\cite{LindbergL15,WangI18} to deepen the understanding of global feasible regions. Regarding exploration, level-set or contour estimation techniques have been adapted to locate unknown feasible regions~\cite{RanjanBM08,BectGLPV12,BachocCG21}. Yet, these methods can be excessively aggressive, thereby hindering optimization progress.

%% file: preliminary.tex
\section{Preliminaries of CBO}
\label{sec:preliminary}

Conventional BO starts from uniformly sampling a set of solutions according to a space-filling experimental design method. Thereafter, it sequentially updates its next sample until the given computational budget is exhausted. BO consists two main components: $i$) a \underline{surrogate model} for approximating the true expensive objective function; and $ii$) an \underline{infill criterion} (based on the optimization of an acquisition function) for deciding the next point of merit.

\subsection{Surrogate Model}
\label{sec:surrogate}

Given a set of training data $\mathcal{D}=\left\{\left(\mathbf{x}^i,f(\mathbf{x}^i)\right)\right\}_{i=1}^N$, we apply the GPR model to learn a Guassian process $\tilde{f}(\mathbf{x})$ with a prior mean function $m(\mathbf{x})$ and a noise-free likelihood~\cite{GPML}. For a candidate solution $\tilde{\mathbf{x}}$, the mean and variance of the target $f(\tilde{\mathbf{x}})$ can be predicted as:
\begin{equation}
    \begin{split}
        \mu_f(\tilde{\mathbf{x}})&= m(\tilde{\mathbf{x}}) + {\mathbf{k}^\ast}^\top K^{-1} \mathbf{f},\\
        \sigma^2_f\left(\tilde{\mathbf{x}}\right)&=k(\tilde{\mathbf{x}},\tilde{\mathbf{x}})-{\mathbf{k}^\ast}^\top K^{-1} {\mathbf{k}^\ast},
    \end{split}
    \label{eq:GP}
\end{equation}
where $\mathbf{k}^\ast$ is the covariance matrix between $X$ and $\tilde{\mathbf{x}}$, $K$ is the covariance matrix of $X$, $X=\left(\mathbf{x}^1,\ldots,\mathbf{x}^N\right)^\top$, and $\mathbf{f}=\left(f(\mathbf{x}^1) - m(\mathbf{x}^1),\ldots,f(\mathbf{x}^N) - m(\mathbf{x}^N)\right)^\top$. In this paper, we use the Mat\'ern ${5/2}$ as the kernel function combined with the constant mean function for all GP models by default. As for the $i$-th constraint in \eqref{eq:cop}, it will be modeled by an independent GP model $\tilde{g}_i(\tilde{\mathbf{x}})$ whose predictive mean and variance are denoted by $\mu_g^i(\tilde{\mathbf{x}})$ and $\sigma^i_g(\tilde{\mathbf{x}})$ respectively.

\subsection{Infill Criterion}
\label{sec:infill}

Instead of directly working on $\tilde{f}(\mathbf{x})$, the actual search process of BO is driven by an acquisition function that naturally strikes a balance between exploitation of the predicted optimum and exploration regarding uncertainty. This paper applies the widely used EI to serve this purpose:
\begin{equation}
    \mathrm{EI}(\tilde{\mathbf{x}}|\mathcal{D})=\sigma_f(\tilde{\mathbf{x}})\big(z\Phi_f\left(z\right)+\phi_f\left(z\right)\big),
    \label{eq:ei}
\end{equation}
where $z=\frac{f^\ast_\mathcal{D}-\mu_f(\tilde{\mathbf{x}})}{\sigma_f(\tilde{\mathbf{x}})}$, $f^\ast_\mathcal{D}=\underset{\left(\mathbf{x},f(\mathbf{x})\right)\in\mathcal{D}}{\min}f(\mathbf{x})$, $\Phi_f$ and $\phi_f$ denote the cumulative distribution function and probability density function according to $\tilde{f}$, respectively.

To tackle unknown constraints, EIC was proposed as a product of the EI with POF~\cite{GardnerKZWC14}:
\begin{equation}
    \mathrm{EIC}(\tilde{\mathbf{x}}|\mathcal{D})=\mathrm{EI}(\tilde{\mathbf{x}}|\mathcal{D}) \cdot \mathrm{POF}(\tilde{\mathbf{x}}),
    \label{eq:eic}
\end{equation}
with
\begin{equation}
    \mathrm{POF}(\tilde{\mathbf{x}})=\mathbb{P}\left[\vec{g}(\tilde{\mathbf{x}})\leq \lambda\right]=\prod_{i=1}^m \Phi_g^i(\lambda),
    \label{eq:pof}
\end{equation}
where $\Phi^i_g$ denotes the cumulative distribution function of the $i$-th constraint based on a GPR model $\tilde{g}_i(\tilde{\mathbf{x}}) \sim \mathcal{N}(\mu_g^i(\tilde{\mathbf{x}}), \sigma_g^{i\;2}(\tilde{\mathbf{x}}))$, $\lambda$ is the threshold of a feasible level and is set to a constant $0$ in this paper.

%% file: proposal.tex
\section{Proposed Method}
\label{sec:method}

\begin{algorithm}[t!]
    \KwIn{Initial dataset
    $\mathcal{D}=\left\{\left(\mathbf{x}^i,f(\mathbf{x}^i), \vec{g}(\mathbf{x}^i)\right) \right\}_{i=1}^{N_0}$
    , budget $N$, and priors of GPs}
    \KwOut{The optimal feasible objective $f_\mathcal{D}^*$}
    
    \For{$k\leftarrow 1$ \KwTo $N$}{
        Build a GPR model for the black-box objective\;
        \For{$i\leftarrow 1$ \KwTo $m$}{
            $\triangleright$ Update $\mathbf{g}_i$ in $\mathcal{D}$ with modified observations $\tilde{\mathbf{\mu}}_g^i$ and $\tilde{\Sigma}_g^i$ using ~\pref{eq:posteriorEP}\;
            $\triangleright$ Build an HLGP model based on $\mathbf{g}_i$\;
        }

        $\triangleright$  ${\mathbf{x}^k} \leftarrow \arg \max_{\tilde{\mathbf{x}} \in \Omega} \mathrm{EICB}(\tilde{\mathbf{x}}\vert\mathcal{D})$\;
        \eIf{$\mathbf{x}^k$ is feasible}{
            $\mathcal{D} \leftarrow \mathcal{D} \cup \left\{\left(\mathbf{x}^k, f(\mathbf{x}^k), \vec{g}(\mathbf{x}^k)\right)\right\}$\;
        }{
            $\triangleright$ $g_i^k \leftarrow +1$ for the $i$-th violated constraints\;
            $\vec{g}(\mathbf{x}^k) = (g_1^k, \dots g_m^k)$\;
            $\triangleright$ $\mathcal{D} \leftarrow \mathcal{D} \cup \left\{\left(\mathbf{x}^k, \mathrm{Null}, \vec{g}(\mathbf{x}^k)\right)\right\}$\;
        }
        }
    \caption{Pseudo code of \our}
    \label{alg:cbob}
\end{algorithm}


This section delineates the implementation of our method, the CBO with balance (dubbed \our), for POCOPs. As shown in~\pref{alg:cbob}, \our\ adheres to the conventional CBO procedure while introducing two unique algorithmic components (highlighted by {$\triangleright$}). The first is a framework for designing an acquisition function, facilitating balanced exploration by effectively harnessing the surrogate models of constraints. The second is a bespoke GP model, specifically formulated to model constraints using partial observations.

\subsection{A Framework for Acquisition Function Design}
\label{sec:eicb}

Under the condition of partial observations, the EI function only updates upon evaluation of a feasible solution, while the POF predominantly targets known feasible regions. This results in an overemphasis on known feasible regions by the EIC, particularly when tackling POCOPs. Inspired by~\cite{Picheny14}, we posit that prioritizing search towards less explored regions can enhance exploratory capability, thus promoting more global search behaviors in a CBO method. This approach has been empirically substantiated in~\cite{LindbergL15,LamW17,ZhangZF21}. In this work, we propose a dynamic version of POF (DPOF) that incorporates an additional exploration capability, rather than prioritizing the most uncertain region indiscriminately, as follows:
\begin{equation}
     \mathrm{DPOF}(\tilde{\mathbf{x}})=\prod_{i=1}^m
     \mathrm{Proj}_{[0,1]}\big[(\rho^i(\tilde{\mathbf{x}})+1)\Phi_g^i(\lambda)\big],
    \label{eq:dpof}
\end{equation}
where $\mathrm{Proj}$ clips values outside $[0, 1]$ to the boundaries, and $\rho^i$ denotes a general exploration function defined below.
\begin{definition}[Exploration function]
    \label{def:explorationfunction}
    A smooth function $\rho^i (\mathbf{x}):\Omega\to[0,\bar\rho]$ is a valid exploration function if it is bounded by $\bar \rho > 0$ and $\rho^i(\tilde{\mathbf{x}})=0,\forall \sigma^i_g(\mathbf{x})=0$.
\end{definition}
With an exploration function $\rho^i$, DPOF assigns more weights to unknown regions than POF to facilitate a global search. However, to prevent excessive exploration and maintain a high probability of obtaining feasible solutions, we multiply $\Phi_g^i(\lambda)$ with $\rho^i$ in~\pref{eq:dpof}. Alternatively, this could be viewed as introducing a dynamic constraint threshold $\lambda(\tilde{\mathbf{x}})=\Phi_g^{i\;-1}(\mathrm{DPOF}^i(\tilde{\mathbf{x}}))$ that varies across different candidate solutions, where $\Phi_g^{i\;-1}$ denotes the inverse function of $\Phi_g^{i}$ and $\mathrm{DPOF}^i(\tilde{\mathbf{x}})$ is the $i$-th factor in~\pref{eq:dpof}. Note that when $\rho^i\equiv 0$, DPOF simplifies to the traditional POF. Building on this, we introduce a new acquisition function, termed as EI with constraint and balance (EICB), formulated as a product of EI and DPOF:
\begin{equation}
    \mathrm{EICB}(\tilde{\mathbf{x}}\vert\mathcal{D})=\mathrm{EI}(\tilde{\mathbf{x}}|\mathcal{D})\cdot \mathrm{DPOF}(\tilde{\mathbf{x}}).
    \label{eq:eicb}
\end{equation}

\begin{theorem}
    \label{theorem:1}
    Assume that the constraint values are fully observable. Assume also that the involved GPs are non-degenerate and satisfy the no-empty-ball property \cite{VazquezB10}. Let $\mathcal{D}$ be the collected observations with $\left(\mathbf{x}^1,f(\mathbf{x}^1)\right)$ fixed in $\chi$ while $\left\{\left(\mathbf{x}^i,f(\mathbf{x}^i)\right)\right\}_{i=2}^N$ are sequentially chosen by 
    \begin{equation}
        \mathbf{x}^{i} = \arg\max_{\mathbf{\tilde x} \in \Omega}\mathrm{EICB}(\tilde{\mathbf{x}}\vert\mathcal{D}).
        \label{eq:eicb_sequential}
    \end{equation}
    Then, as $N \to \infty$, almost surely:
    \begin{enumerate}
        \item the acquisition function $\sup_{\tilde{\mathbf{x}}\in \Omega} \mathrm{EICB}(\tilde{\mathbf{x}}\vert\mathcal{D}) \to 0$;
        \item the evaluated best objective $f_\mathcal{D}^\ast\to f_\chi^\ast$;
    \end{enumerate}
    where $f^\ast_\chi$ represents the global optimum of problem \eqref{eq:cop}.
\end{theorem}
The proof of~\pref{theorem:1} is sketched in \hyperref[sec:theoretical_eicb]{Section A} of the supplementary document. This theorem suggests that the incorporation of $\rho^i$ as designed in~\pref{eq:dpof} does not undermine the asymptotic convergence capability of EI-based acquisition functions, such as EICB. In the following subsection, we propose an instance of the exploration function under~\pref{def:explorationfunction} that outperforms the EIC in terms of efficiently conducting global optimization for POCOPs.

\subsubsection{An instance of the exploration function}
\label{sec:design_ecib}

In the context of EICB framework, exploration during optimization can be facilitated by an apt design of $\rho^i$. In this paper, we concentrate on identifying promising constraint boundaries, as opposed to aggressively targeting the most uncertain regions, a tactic often employed in level-set estimation and active learning~\cite{RanjanBM08,BichonESMM08,BectGLPV12,BachocCG21}. To this end, we first define a utility function representing the potential of being the boundary (POB) for the $i$-th constraint at $\tilde{\mathbf{x}}\in\Omega$ as follows:
\begin{equation}
    \mathrm{POB}^i(\tilde{\mathbf{x}})=\left\{ 
        \begin{array}{cc}
            1, & \tilde{g}_i(\tilde{\mathbf{x}}) \in \left[-\varepsilon(\tilde{\mathbf{x}}), \; \varepsilon(\tilde{\mathbf{x}}) \right], \\
            0, & \mathrm{otherwise},
        \end{array}
    \right.
\label{eq:pob}
\end{equation}
where $\varepsilon(\tilde{\mathbf{x}})=\beta{\sigma}^i_g(\tilde{\mathbf{x}})$ and $\beta>0$ represents a confidence level. Taking the expectation of~\pref{eq:pob} over the predicted distribution of $\tilde g_i(\tilde{\mathbf{x}})$ and defining $\bar g_i(\tilde{\mathbf{x}})=\mu_g^i(\tilde{\mathbf{x}})/\sigma_g^i(\tilde{\mathbf{x}})$, we obtain a valid exploration function as:
\begin{equation}
    \rho^i(\tilde{\mathbf{x}})= \Phi\left(\beta - \bar g_i (\tilde{\mathbf{x}})\right)-\Phi\left(- \beta - \bar g_i (\tilde{\mathbf{x}}) \right),
    \label{eq:exploration_function}
\end{equation}
where $\Phi$ is the cumulative distribution function of $\mathcal{N}(0, 1)$.

\begin{figure}[thb]
    \centering
    \scalebox{.8}{\includegraphics[width=\linewidth]{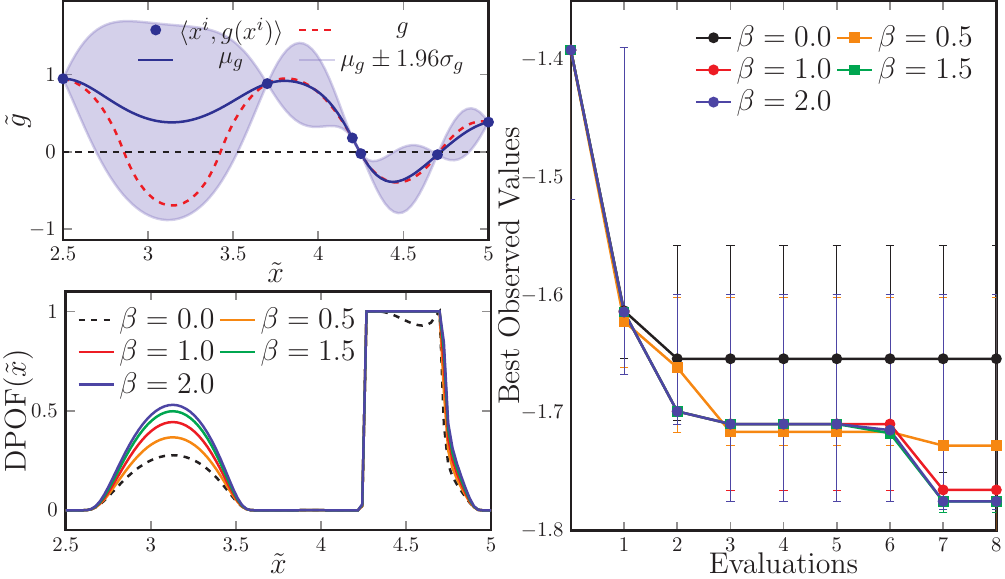}}
    \caption{Illustration of EICB with~\pref{eq:exploration_function} by a toy $1$-D example (please refer to \hyperref[sec:empirical_eicb]{Section B.1} of the supplementary document for more details). (Left up) The constraint surrogate model with $6$ observed data pairs where the red dotted line denotes the true function. (Left down) The DPOF with different $\beta$ on $[2.5, 5]$, where the dotted line with $\beta=0$ reduces to POF. (Right) Optimization trajectories on $[0, 10]$ with $5$ randomly repeated experiments.}
    \label{fig:illustrationEICB}
\end{figure}

The illustrative example in~\pref{fig:illustrationEICB} demonstrates how DPOF, given a surrogate model for a constraint, assigns more weight to the unknown feasible region ($[2.85,3.45]$ in this case) as $\beta$ increases. For a previously located feasible region such as $[4.3,4.7]$, DPOF provides equal weights (approximately $1$) to all candidates within this region when $\beta\ge 0.5$. In contrast, POF ($\beta=0$) assigns differentiated weights based on different $\Phi_g(0)$ values. By refining the boundary, DPOF ensures a more balanced weight distribution within the located feasible region, hence the nomenclature, EICB. We posit that this balanced approach enhances the decision-making capabilities of EI, as compared to the imbalanced weights scenario posed by POF. As a positive consequence, EICB encourages greater exploration towards unknown regions. As $\rho^i$ in~\pref{eq:exploration_function} is bound by $1$, DPOF can assign a maximum of $2\Phi_g(0)$ to any given candidate. Empirically, this subtle adjustment leads to enhanced optimization efficiency as evidenced in~\pref{fig:illustrationEICB}, thanks to the introduction of exploration. Besides, we provide a more aggressive design of $\rho^i$ that bolsters the reduction of uncertainty in global feasible regions. Due to the space restriction, this discussion is presented in \hyperref[sec:empirical_emub]{Section B.2} of the supplementary document. 

\subsection{Surrogate Models Under Partial Observations}
\label{sec:hlgp}

\begin{figure}[thb]
    \centering
    \scalebox{.8}{\includegraphics[width=\linewidth]{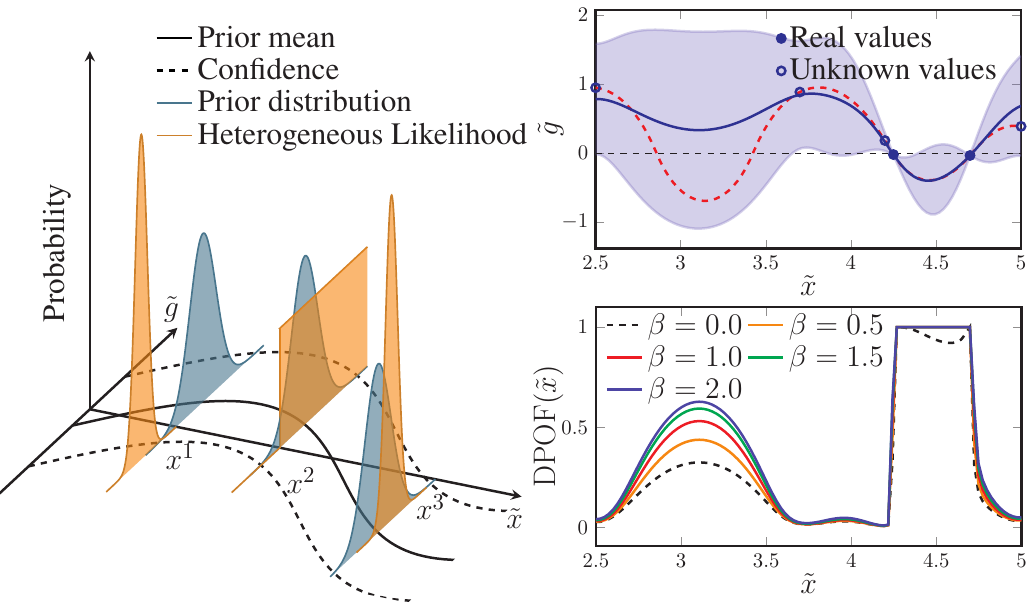}}
    \caption{Illustration of HLGPs. (Left) Three observations with heterogeneous likelihood distributions, including a truncated distribution on $x^2$, and Gaussian distributions on $x^1$ and $x^3$. (Right up) An EP-based HLGP model with partial observations. (Right down) The DPOF with different $\beta$.}
    \label{fig:HLGP_EP}
\end{figure}

When dealing with partially observable constraints, observations are composed of two distinct aspects: $i$) the actual values associated with feasible solutions, and $ii$) the truncated distribution of possible values for all solutions, such as $\mathbb{P}(\vec{g}>0) = 1$, as depicted in~\pref{fig:HLGP_EP}. The simultaneous consideration of these two types of observations can be achieved by attaching individual likelihood distributions to feasible/infeasible solutions, as is done in HLGP.

For the $i$-th constraint, the posterior of a latent function, $p(\tilde{\mathbf{g}}_i \vert X, \mathbf{g}_i)$, within an HLGP model is determined via the Bayes rule using the prior distribution $p(\tilde{\mathbf{g}}_i\vert X)=\mathcal{N}\left(\mathbf{0}, K\right)$ in~\pref{eq:GP}, along with the individual likelihood distributions. Specifically, the  likelihood can be expressed as:
\begin{equation}
    p({g}^k_i  \vert \tilde{g}_i^k)=
    \begin{dcases}
        \mathcal{N}\left(g_i(\mathbf{x}^k),\sigma^2\right), &\text{if } g_i(\mathbf{x}^k)\leq 0, \\
         \Phi(\alpha^{-1} g_i( \mathbf{x}^k)), &\text{if } g_i(\mathbf{x}^k)>0,
    \end{dcases}
    \label{eq:Helikelihood}
\end{equation}
where $k\in\{1,\ldots,N\}$, $\mathbf{g}_i=(g_i^1,\ldots,g_i^N)^\top$ represents $N$ observations of $g_i(X)$, $\tilde{\mathbf{g}}_i=(\tilde{g}_i^1,\ldots, \tilde{g}_i^N)^\top$ denotes $N$ latent functions, $\sigma\ge 0$ stands for the noise level, and $\alpha > 0$ is a scaling parameter. We set $\sigma=10^{-6}$ to indicate a noise-free environment and $\alpha=10^{-6}$ to approximate the truncating step function that implies $\mathbb{P}(g_i(\mathbf{x}^k)>0)=1$, $\forall g_i(\mathbf{x}^k)>0$~\cite{RiihimakiV10}.

\subsubsection{HLGP Inference via Expectation Propagation}
\label{sec:ep}

In this paper, we employ expectation propagation (EP)~\cite{Minka01}, a principled and highly efficient approach to handle non-Gaussian likelihoods. This provides Gaussian approximations to both the posterior and predicted distributions of HLGP. First, the posterior is formulated as:
\begin{equation}
    p(\tilde{\mathbf{g}}_i \vert X, \mathbf{g}_i) =
    \frac{1}{Z} p(\tilde{\mathbf{g}}_i\vert X) 
    \prod_{k=1}^N p( {g}^k_i\vert \tilde{g}_i^k),
    \label{eq:posteriorGP}
\end{equation}
where $Z$ is the normalization factor. For the $k$-th observation $g_i^k$, EP assigns it an un-normalized Gaussian distribution $t_i^k \triangleq \tilde{Z}_i^k\mathcal{N}\left( \tilde\mu_i^k, \tilde\sigma_i^{k\;2} \right)$ to locally approximate its exact likelihood. In this vein, the posterior is approximated by:
\begin{gather}
    p(\tilde{\mathbf{g}}_i \vert X, \mathbf{g}_i) \approx  \frac{1}{Z_{\mathrm{EP}}} p(\tilde{\mathbf{g}}_i\vert X) \prod_{k=1}^N t_i^k = \mathcal{N}(\bm{\mu}_g^i, \Sigma_g^i) \nonumber \\
    \text{with }~\bm{\mu}_g^i= \Sigma_g^i {\tilde{\Sigma}_{g}}^{i\;-1} \tilde{\bm{\mu}}_g^i ~\text{ and }~ \Sigma_g^i = (K + {\tilde{\Sigma}_{g}}^{i\;-1})^{-1},
    \label{eq:posteriorEP}
\end{gather}
where $\tilde{\bm{\mu}}_g^i = \left(\tilde\mu_i^1,\cdots,\tilde\mu_i^N\right)^\top$, $\tilde\Sigma_g^i$ denotes a diagonal matrix with the $k$-th element being $\tilde{\sigma}_i^{k\;2}$, and $Z_{\mathrm{EP}}$ is the marginal likelihood. The site parameters in $t_i^k$ of a Gaussian likelihood in~\pref{eq:Helikelihood} are valued by $\tilde{Z}_i^k = 1$, $\tilde{\mu}_i^k = g_i(\mathbf{x}^k)$, $\tilde{\sigma}_i^k = \sigma$. Differently, the site parameters of a non-Gaussian likelihood in~\pref{eq:Helikelihood} should be computed by the moment matching \cite{RiihimakiV10}. Detailed formulations of this part are delineated in \hyperref[sec:ep_app]{Section C} of the supplementary document. For a candidate solution $\tilde{\mathbf{x}}$, the mean and variance of the HLGP model $\tilde g_i(\tilde{\mathbf{x}})$ are predicted as:
\begin{equation}
    \begin{aligned}
        \mu_g^i(\tilde{\mathbf{x}})&= m(\tilde{\mathbf{x}}) + {\mathbf{k}^\ast}^\top(K+\tilde{\Sigma}_g^i)^{-1}\tilde{\bm{\mu}}_g^i, \\
        \sigma_g^{i\;2} \left(\tilde{\mathbf{x}}\right)&=k(\tilde{\mathbf{x}},\tilde{\mathbf{x}})-{\mathbf{k}^\ast}^\top (K+\tilde{\Sigma}_g^i)^{-1} {\mathbf{k}^\ast}.
    \end{aligned}
    \label{eq:predictionEP}
\end{equation}

In principle, the hyperparameters of EP-based GP models should be updated by maximizing the marginal likelihood $Z_{\mathrm{EP}}$. Since~\eqref{eq:predictionEP} resembles~\eqref{eq:GP}, from another perspective, EP algorithm serves as a data generator for HLGPs, i.e., assigning \emph{virtual observations} for infeasible solutions with an estimation of noise levels. Accordingly, the hyperparameters of an HLGP model can be optimized by maximizing the marginal likelihood of a vanilla GPR model using these injected observations rather than $Z_{\mathrm{EP}}$ for better computation efficiency, as noted in \cite{GPML}.

\subsubsection{Comparison with GPC}
\label{sec:gpc}

\begin{figure}[thb]
    \centering
    \scalebox{.8}{\includegraphics[width=\linewidth]{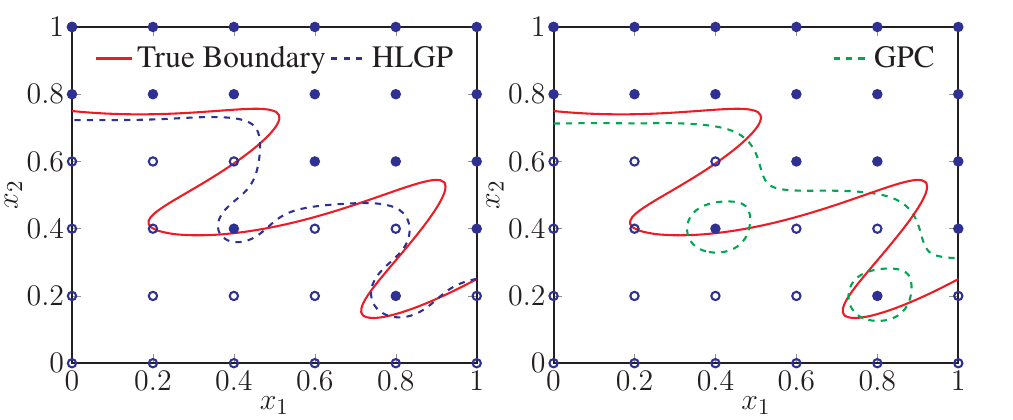}}
    \caption{Different curves of feasible boundary predicted by (Left) an HLGP model versus (Right) a GPC model.}
    \label{fig:HLGP_GPC}
\end{figure}

Finally, we present a brief comparison between the HLGP models in our proposed method and the GPC models frequently used in existing CBO algorithms for unobservable constraints. Notably, executing CBO routines is also feasible by only considering the truncated value distribution in~\pref{eq:Helikelihood}. We contend that HLGP, by leveraging available observations, constructs a more reliable surrogate model than GPC. As illustrated in~\pref{fig:HLGP_GPC}, $36$ equidistant solutions are evaluated, with constraint values being unobservable for half. The true boundary and the predictions made by HLGP and GPC are provided for performance assessment. The GPC model delineates the boundary by maximizing the distance between the nearest distinct solutions, mirroring the approach of an SVM~\cite{GPML}. This results in the prediction of three disconnected feasible regions, deviating from the true feasible region. In contrast, HLGP predicts a connected feasible region that covers the majority of the true one. As such, we anticipate that HLGP, with its superior modeling capability, can enhance optimization for POCOPs.

\subsection{Summary of the \our\ Algorithm}
\label{sec:summary}

To improve the optimization efficiency for POCOPs, \our\ fully exploits available observations by re-designing a more balanced acquisition function and constructing principled surrogate models from mixed observations. Our method combines a boundary-based exploration function, as in~\pref{eq:exploration_function}, and non-informative likelihoods as in~\pref{eq:Helikelihood}. Moreover, \our\ preserves ample flexibility in designing exploration functions and likelihoods, thus making it adaptable to individual optimization problems and suitable for further investigations.
\begin{itemize}
    \item Building on the concept of exploring potentially feasible regions as~\cite{LindbergL15,WangI18}, \our\ introduces extra exploration during optimization but with robust theoretical support underpinning this design. Moreover, our EICB maintains the computational efficiency of the original EIC, unlike the methods proposed in~\cite{GelbartSA14,LamW17,ZhangZF21}.

    \item Inspired by the level-set estimation methods~\cite{BachocCG21}, we propose an innovative exploration function as~\pref{eq:exploration_function} that emphasizes potential boundaries, integrating with the POF for efficient optimization. This design avoids aggressively evaluating unknown regions, which is a common tactic in active learning~\cite{RanjanBM08,BichonESMM08,BectGLPV12}.

    \item By employing HLGPs, we build a better surrogate model for each partially observable constraint function, outperforming GPC-based methods~\cite{LindbergL15,PerroneSJAS19,AriafarCBD19,BachocHP20,Candelieri21}. With the aid of EP and its generated virtual observations, we manage to construct HLGP models with commendable computational efficiency, resulting in an improvement over other models with mixed observations~\cite{PourmohamadL16,ZhangDL19}.
\end{itemize}

%% file: settings.tex
\section{Experiment Setup}
\label{sec:setting}

In this section, we present the experimental settings used in our empirical study.

\subsection{Benchmark Suite}
\label{sec:problems}


Our experiments consider various optimization tasks, including synthetic problems, engineering design cases, hyperparameter optimization (HPO) problems based on scikit-learn \cite{PedregosaVGMTGBPWDVPCBPD11}, and reinforcement learning tasks based on Open AI Gym \cite{BrockmanCPSSTZ16}, to constitute our benchmark suite. In addition, we consider the following two scenarios of POCOPs.
\begin{itemize}
    \item\underline{S1:} The first scenario is that only $f$ is partially observable. Specifically, problems include $10$D Keane's bump function (KBF)~\cite{Keane94}, $4$D welded beam design (WBD)~\cite{Deb00}, $7$D HPO of XGBoost on the California housing dataset with a model size constraint (XGB-H), and $12$D Lunar Landing with an energy constraint (Lunar)~\cite{ErikssonPGTP19}.
    \item\underline{S2:} The second scenario considers both $f$ and $\vec{g}$ are partially observable. Specifically, problems include $10$D Ackley function with one constraint (Ackley)~\cite{MarcoBKHRT21}, $4$D pressure vessel design (PVD)~\cite{CoelloM02}, $8$D HPO of MLP on the digits dataset with a model size constraint (MLP-D), and $16$D Swimmer with an energy constraint (Swimmer)~\cite{WangFT20}.
\end{itemize}

\subsection{Peer Algorithms}
\label{sec:peers}

We consider three state-of-the-art CBO methods, including \texttt{EIC}~\cite{GardnerKZWC14} in the EI-based family, min-value entropy search with constraints (\texttt{MESC})~\cite{TakenoTSK22} in the information-theoretic family, and Thompson sampling with constraints (\texttt{TSC})~\cite{ErikssonP21} in the stochastic sampling family. For S2, note that both \texttt{EIC} and \texttt{MESC} can handle binary constraint feedback \cite{BachocHP20,PerroneSJAS19}. For S1, all algorithms use GPR to build the surrogate models. For S2, we choose either HLGP or GPC for modeling constraints. In particular, we use a dedicated subscript to represent the corresponding surrogate model, e.g., \texttt{EIC}$_c$ and \texttt{EIC}$_{h}$ denote EIC with GPC and HLGP, respectively.

\subsection{General Settings}
\label{sec:parameters}

All algorithms are implemented according to their open-source code \cite{ErikssonP21,TakenoTSK22}. 
In \texttt{MESC}, the optimal solutions are sampled $20$ times. Both \texttt{MESC} and \texttt{TSC} sample with a grid size of $1\,000$. For \texttt{CBOB} with ~\pref{eq:exploration_function}, we fix $\beta=1.96$ to obtain a $95\%$ confidence level. As~\pref{eq:exploration_function} is a conservative design for exploration, we omit the study on more conservative behaviors with smaller $\beta$. 
Each experiment is independently repeated $20$ times with shared random seeds. For all tasks, the Sobol sequence is used to generate $11\times n$ initial samples, then $100$ function evaluations (FEs) are performed in each experiment. Detailed settings of all algorithms and benchmark problems are presented in \hyperref[sec:experiment]{Section D} of the supplemental document. The source code of our project is available\footnote{\href{https://github.com/COLA-Laboratory/CBOB}{https://github.com/COLA-Laboratory/CBOB}}.

%% file: experiments.tex
\section{Experiment Results}

The optimization trajectories of all experiments are given in Figures~\ref{fig:exp_s1} and~\ref{fig:exp_s2}. In addition, the median best-evaluated values (BOVs) and average ratios of feasible evaluations (ROFs) of different algorithms are presented in Tables~\ref{tab:s1} and~\ref{tab:s2}. We empirically study the efficacy of \texttt{CBOB} from three aspects: $i)$ the \emph{improvements} on \texttt{CBOB} for EI-based CBO methods and GPC-based models; $ii)$ the \emph{competitiveness} of \texttt{CBOB} with other peer algorithms; and $iii)$ the \emph{relationship} between efficiency and exploration ability of \texttt{CBOB}.

\begin{figure*}[thb]
    \includegraphics[width=\linewidth]{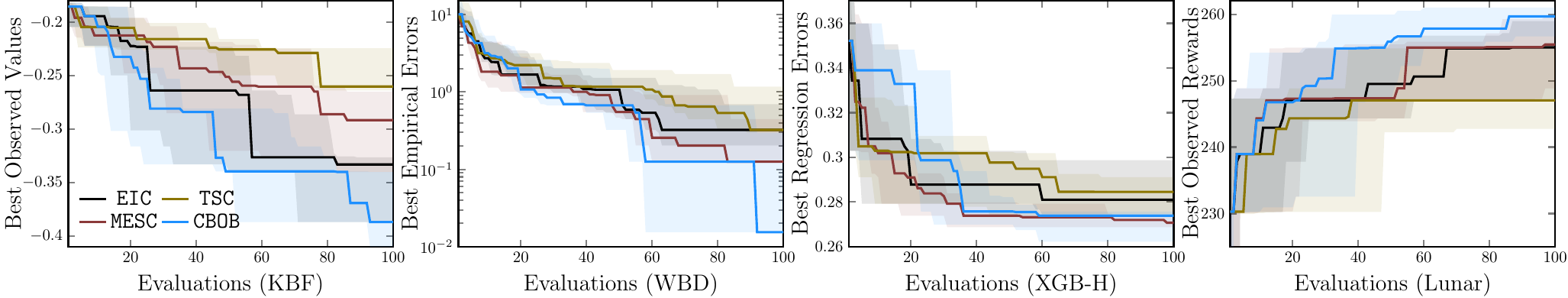}
    \caption{Optimization trajectories of different tasks in S1 ($f$ is partially observable).}
    \label{fig:exp_s1}
\end{figure*}

\begin{figure*}[thb]
    \includegraphics[width=\linewidth]{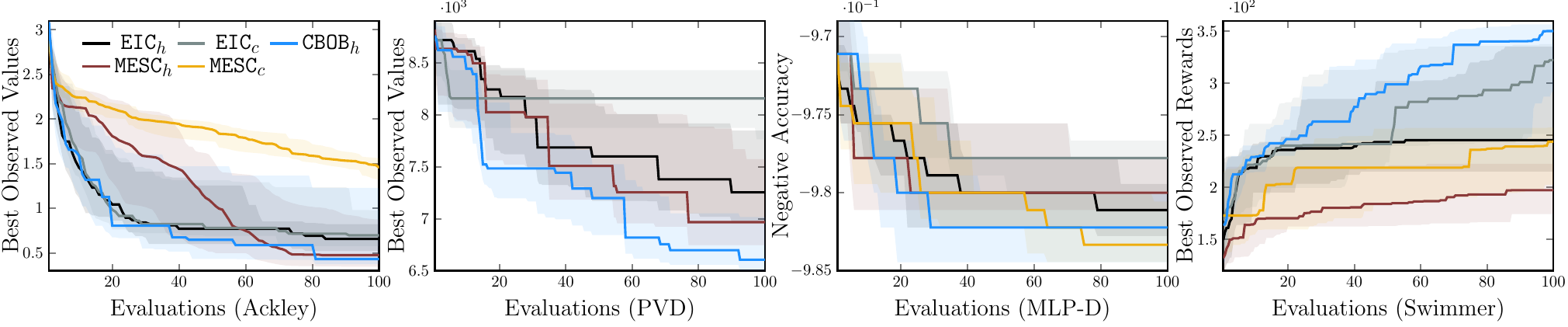}
    \caption{Optimization trajectories of different tasks in S2 (both $f$ and $\vec{g}$ are partially observable).}
    \label{fig:exp_s2}
\end{figure*}

\begin{table}[tb]
    \centering
    \setlength{\tabcolsep}{9mm}{
    \begin{tabular}{|c|c|cccc|}
    \hline
     \multicolumn{2}{|c|}{Algorithm} & KBF & WBD & XGB-H & Lunar \\
     \hline
     \multirow{2}*{\texttt{EIC}} & BOV & $-0.33$ & $2.47$ & $0.281$ & $255$ \\
     ~ & ROF & $99.3\%$ & $75.3\%$ & \underline{$22.4\%$} & $82.3\%$\\
     \hline
     \multirow{2}*{\texttt{MESC}} & BOV  & $-0.29$ & $2.28$ & $\mathbf{0.271}$ & $255$ \\
     ~ & ROF & $99.5\%$ & \underline{$18.3\%$} & $58.1\%$ & $84.5\%$ \\
     \hline
     \multirow{2}*{\texttt{TSC}} & BOV  & $-0.26$ & $2.47$ & $0.283$ & $247$\\
     ~ & ROF & $99.8\%$ & $77.9\%$ & $51.2\%$ & $87.7\%$ \\
     \hline
     \multirow{2}*{\our} & BOV  & $\mathbf{-0.39}$ & $\mathbf{2.16}$ & $0.274$ & $\mathbf{260}$\\
     ~ & ROF & \underline{$99.2\%$} & $67.4\%$ & $25.8\%$ & \underline{$70.0\%$}\\
     \hline
    \end{tabular}}
    \caption{The BOV and ROF of different algorithms in S1.}
    \label{tab:s1}
\end{table}

\begin{table}[tb]
    \centering
    \setlength{\tabcolsep}{8mm}{
    \begin{tabular}{|c|c|cccc|}
    \hline
     \multicolumn{2}{|c|}{Algorithm} & Ackley & PVD & MLP-D & Swimmer \\
     \hline
     \multirow{2}*{\texttt{EIC}$_c$} & BOV & $0.70$ & $8155$ & $0.978$ & $322$ \\
     ~ & ROF & $77.9\%$ & $5.12\%$ & $84.5\%$ & $82.3\%$\\
     \hline
     \multirow{2}*{\texttt{MESC}$_c$} & BOV & $1.44$ & N/A & $\mathbf{0.983}$ & $244$ \\
     ~ & ROF & $78.2\%$ & N/A & $83.1\%$ & $83.5\%$ \\
     \hline
     \multirow{2}*{\texttt{EIC}$_h$} & BOV & $0.66$ & $7598$ & $0.981$ & $245$ \\
     ~ & ROF & $38.4\%$ & \underline{$3.13\%$} & \underline{$64.8\%$} & \underline{$67.2\%$}\\
     \hline
     \multirow{2}*{\texttt{MESC}$_h$} & BOV  & $0.48$ & $7507$ & $0.98$ & $197$ \\
     ~ & ROF & $54.1\%$ & $3.78\%$ & $50.1\%$ & $79.8\%$ \\
     \hline
     \multirow{2}*{\our} & BOV  & $\mathbf{0.43}$ & $\mathbf{7198}$ & $0.982$ & $\mathbf{350}$\\
     ~ & ROF & \underline{$38.0\%$} & $3.32\%$ & $65.5\%$ &  $77.1\%$\\
     \hline
    \end{tabular}}
    \caption{The BOV and ROF of different algorithms in S2.}
    \label{tab:s2}
\end{table}

\begin{figure}[thb]
    \centering
    \scalebox{.8}{\includegraphics[width=\linewidth]{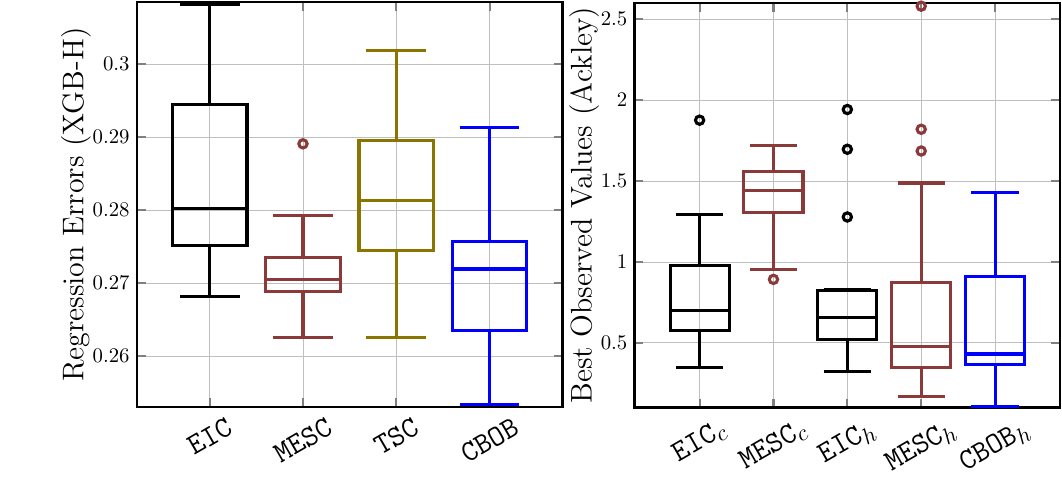}}
    \caption{Box plot of the final best observed values of different algorithms on Ackley and XGB-H.}
    \label{fig:exp_box}
\end{figure}

\subsection{Improvements} 
Although EIC can be more efficient within $10$ to $20$ FEs, such as in WBD, XGB-H, Ackley, and MLP-D, EICB outperforms EIC in all experiments after $100$ FEs, which demonstrates the design of DPOF. As the natural expense of exploration, the value deviation of EICB may be larger during the search. As given in ~\pref{fig:exp_box}, despite \texttt{EIC}$_c$ and \texttt{EIC}$_h$ have smaller deviations of the best evaluated values, EICB obtains a more promising result in a statistical sense. In addition, compared to GPC, HLGP does not always improve \texttt{EIC} and \texttt{MESC}, whereas benefiting \our\ well.

\subsection{Competitiveness}
The \our\ shows a strong competitiveness in all experiments against other CBO methods. In addition to the problems that EI-based methods perform well, such as KBF and Swimmer, \our\ remains competitive in problems that EI-based methods struggle, such as XGB-H, PVD and MLP-D. In comparison, \texttt{MESC}, as a promising CBO method in most problems, is inefficient in problems such as KBF and Swimmer. The \texttt{TSC} that showed high efficiency in fully observable environments with trust regions \cite{ErikssonP21}, struggles in solving POCOPs. Moreover, \our\ and other EI-based CBO methods show less efficiency in HPO problems, which agrees with the empirical results in \cite{WatanabeH23}.

\subsection{Relationship between efficiency and exploration}
In Tables~\ref{tab:s1} and~\ref{tab:s2}, while \our\ obtains better feasible solutions, its ROFs are relatively low, i.e., more evaluated solutions of \our\ are infeasible. On the one hand, this agrees with the ideas of \cite{ZhangZF21} that exploration towards infeasible regions facilitates optimization efficiency of CBO methods. It also explains the larger deviation of \our\ in~\pref{fig:exp_box} that infeasible evaluations return little information in POCOPs. Besides, we find this exploration effective as the number of outliers of \our\ reduces, such as MESC$_h$ and \our\ in Ackley of ~\pref{fig:exp_box}. On the other hand, since EI has already considered exploration, \texttt{EIC} can also have fewer ROFs during the search. Differently, in order not to break the balance of EI between exploration and exploitation, \our\ assigns more balanced weights for constraint handling. We highlight that this idea can also be integrated with other acquisition functions, such as the probability of improvement and parzen estimator \cite{garnett_bayesoptbook_2023}.

%% file: conclusion.tex
\section{Concluding Remarks}
\label{sec:conclusions}

This paper designs \texttt{CBOB} that fully exploits the available observations for better exploration and surrogate modeling, by both theoretical and empirical analysis, demonstrating that \our\ has the potential to be a promising CBO method for POCOPs. Due to the space restriction, we provide some further discussions on the limitations and potential impact of this work in \hyperref[sec:further]{Section E} of the supplementary document. Further investigations include the in-depth theoretical study of \our\ and the design of the exploration functions that are suitable for individual problems. We will also endeavor to propose a risk-aware improvement of EICB regarding robustness and the reduction of infeasible evaluations, contributing to more practical optimization scenarios.

%% file: tech_appendix.tex
\newcommand{\figuretag}[1]{%
  \addtocounter{figure}{-1}%
  \renewcommand{\thefigure}{#1}%
}

\fancyhead{}

\section*{Supplementary document: Technical Appendices}

This document includes the detailed derivation, algorithm design, and experiment settings in the paper. Specifically,  we present the theoretical analysis of the EICB framework in \hyperref[sec:theoretical_eicb]{Section A}. The empirical study of EICB, including another design of the exploration function and comparative experiments of EI-based CBO methods, is given in \hyperref[sec:empirical]{Section B}. In \hyperref[sec:ep_app]{Section C}, we show a detailed derivation of the EP approximation for HLGP models. The experiment settings are presented in \hyperref[sec:experiment]{Section D}. Finally in \hyperref[sec:further]{Section E}, we discuss the limitations and potential impact of our work.

\section*{Section A. Theoretical analysis of EICB}
\label{sec:theoretical_eicb}

We justify the design of EICB by studying its convergence property under the framework of the \emph{consistency} research for EI~\cite{BectBG19} and vanilla EIC~\cite{BachocHP20}. In the context of sequential design, let $\mathcal{F}_N$ denote the $\sigma$-algebra generated by the random variables $\mathbf{x}^1, Z^1$, $\dots$, $\mathbf{x}^N, Z^N$ where $Z^i$ is the union observation of $\left( f(\mathbf{x}^i), \vec g(\mathbf{x}^i)\right)$. Additionally, let $\mathcal{F}_{N, \tilde{\mathbf{x}}}$ be the $\sigma$-algebra generated by $\mathbf{x}^1, Z^1$, $\dots$, $\mathbf{x}^N, Z^N$, $\tilde{\mathbf{x}}, \tilde Z$ with $\tilde Z$ the union observation of $\left( f(\tilde{\mathbf{x}}), \vec g(\tilde{\mathbf{x}})\right)$.
Then, a sequential design strategy derived from the EICB acquisition framework takes the following form:
\begin{equation}
    \mathbf{x}_{N+1} = \arg\max_{\tilde{\mathbf{x}} \in \Omega} \mathbb{E}_N \left[ M_N^0 - M^\rho_{N, \tilde{\mathbf{x}}}  \right],
    \tag{A.1} \label{eq:eicb_sq_design}
\end{equation}
in which 
\begin{gather}
    M^\rho_{N, \tilde{\mathbf{x}}} = \min_{ 
        \mathbf{x} \in \Omega, \,
        \mathbb{P}(\vec{g}(\mathbf{x})<\lambda(\mathbf{x}) \vert\mathcal{F}_{N, \tilde{\mathbf{x}}}) = 1, \,
        \sigma_f(\mathbf{x} \vert \mathcal{F}_{N, \tilde{\mathbf{x}}})=0 
    } \tag{A.2} 
    \tilde{f} (\mathbf{x}), \\ 
    M_N^0 = \min_{\mathbf{x} \in \Omega, \,
        \mathbb{P}(\tilde{g}(\mathbf{x})<0\vert \mathcal{F}) = 1, \,
       \sigma_f(\mathbf{x} \vert \mathcal{F}_N)=0 } \tilde{f}(\mathbf{x}). \tag{A.3} 
\end{gather}

\begin{definition}[Non-degenerate GPs \cite{GinsbourgerRD16}]
    A non-degenerate GP model $\tilde{f}(\mathbf{x})$ predicts $\sigma_f(\mathbf{x}) = 0$ if and only if $\left( \mathbf{x},  f(\mathbf{x})\right) \in \mathcal{D}$.
\end{definition}

\begin{definition}[No-empty-ball~\cite{VazquezB10}]
    \label{def:neb}
    Let $(\mathbf{x}^N)_{N\ge 1}$ be any sequence in $\Omega$, and $z$ be any solution in $\Omega$. A GP model has the no-empty-ball (NEB) property, 
    if its prediction error at $z$ goes to zero, $\forall \epsilon>0$, there always exists $N\ge 1$ such that $\vert z - \mathbf{x}^N \vert < \epsilon$.
\end{definition}


\begin{proof}[Proof of Theorem 1]
    When GPs are non-degenerate, \pref{eq:eicb_sq_design} becomes equivalent to \pref{eq:eicb_sequential}. Specifically, $M_N^0$ will be equal to $f^*_\mathcal{D}$ in \pref{eq:ei}, while $M^\rho_{N,\tilde{\mathbf{x}}}$ will be the predicted feasible objective value under dynamic constraint threshold implicitly defined by \pref{eq:dpof}. Next, we present the criteria for \textit{asymptotic convergence} of EICB.
    
    The proof of the first statement, i.e., the convergence of EICB, consists of three steps.
    
    \textbf{Step 1. EICB serves as a stepwise uncertainty reduction (SUR) sequential design.} For $N\ge 2$, a minimization version of \pref{eq:eicb_sq_design} can be given as
    \begin{equation}
        \mathbf{x}_{N+1} = \arg\min_{\tilde{\mathbf{x}} \in \Omega} \mathbb{E}_N \left[H_{N, \tilde{\mathbf{x}}}\right], \tag{A.4}
    \end{equation}
    in which
    \begin{align}
        H_{N, \tilde{\mathbf{x}}} = &   M^\rho_{N, \tilde{\mathbf{x}}} - M_N^0 \nonumber \\
        = & \mathbb{E}_{N, \tilde{\mathbf{x}}}\left[  M^\rho_{N, \tilde{\mathbf{x}}} - \min_{\mathbf{x} \in \Omega, \; \vec{g}(\mathbf{x}) < 0} \tilde{f} (\mathbf{x})  \right]. \tag{A.5}
    \end{align}
    The above equation holds since: $i)$ $M^0_N$ is independent from $\tilde{\mathbf{x}}$, and $ii)$ $\mathbb{E}_{N, \tilde{\mathbf{x}}}\left[ M^\rho_{N, \tilde{\mathbf{x}}} \right] = M^\rho_{N, \tilde{\mathbf{x}}}$ for minimum operation. Therefore EICB strategy can be transformed into an equivalent SUR sequential design strategy for $H_{N, \tilde{\mathbf{x}}}$. Likewise, we define
    \begin{equation}
        H_{N} = \mathbb{E}_{N}\left[  M^\rho_{N} - \min_{\mathbf{x} \in \Omega, \; \vec{g}(\mathbf{x}) < 0} \tilde{f} (\mathbf{x})  \right]. \tag{A.6}
    \end{equation}
    
    \textbf{Step 2. $(H_N)$ is a supermartingale.} 
    For well-structured GP models and well-defined smooth functions $\rho^i$, we have: $i)$ $\sigma_f(\mathbf{x} \vert \mathcal{F}_{N+1}) \leq \sigma_f(\mathbf{x} \vert \mathcal{F}_{N})$ (based on the definition of GP predicted variance), and $ii)$ $ \mathbb{P}(\vec{g}(\mathbf{x})<\lambda(\mathbf{x}) \vert\mathcal{F}_{N}) = 1$ is sufficient for $\mathbb{P}(\vec{g}(\mathbf{x})<\lambda(\mathbf{x}) \vert\mathcal{F}_{N+1}) = 1$ (based on the non-increasing property of $\rho^i$ on an evaluated solution ${\mathbf{x}}^N$). Therefore, the following inequality holds:
    \begin{equation}
        H_{N} - \mathbb{E}_N [H_{N+1}] = \mathbb{E}_{N}\left[  M^\rho_{N} - M^\rho_{N+1}\right] \ge 0,
        \tag{A.7}
        \label{eq:proof_martine}
    \end{equation}
    which implies that $(H_N)_{N\in \mathbb{N}}$ is a  supermartingale. According to \cite{BachocHP20}, there is $H_{N} - \mathbb{E}_N [H_{N+1}] \to 0$ as $N\to \infty$, and also 
    \begin{equation}
        \sup_{\tilde{\mathbf{x}} \in \Omega} \Big[H_N - \mathbb{E}_N[H_{N, \tilde{\mathbf{x}}}] \Big] \to 0.
        \tag{A.8}
        \label{eq:proof_martingale}
    \end{equation}
    
    \textbf{Step 3. $\mathrm{EICB}(\tilde{\mathbf{x}}\vert\mathcal{D})$ converges to $0$ almost surely.} Note that for $N$ observed solutions, $M^\rho_{N} = M^0_{N}$ as $\lambda =0$. According to \pref{eq:proof_martine}, we also have
    \begin{align}
        \sup_{\tilde{\mathbf{x}} \in \Omega} \mathbb{E}_{N}\left[  M^\rho_{N} - M^\rho_{N, \tilde{\mathbf{x}}}\right]  \ge    \sup_{\tilde{\mathbf{x}} \in \Omega} \mathbb{E}_{N}\left[  M^0_{N} - M^\rho_{N,  \tilde{\mathbf{x}}}\right] \nonumber \\
        \ge \sup_{\tilde{\mathbf{x}} \in \Omega} \mathrm{EI}(\tilde{\mathbf{x}}|\mathcal{D})\cdot \mathrm{DPOF}(\tilde{\mathbf{x}}).
        \tag{A.9}
        \label{eq:proof_eicb_to0}
    \end{align}
    Therefore, with the same proof as that of Proposition 2.9 \cite{BectBG19}, for $N \to \infty$, \eqref{eq:proof_martingale} and \eqref{eq:proof_eicb_to0} yield $\sup_{\tilde{\mathbf{x}} \in \Omega} \mathrm{EICB}(\tilde{\mathbf{x}}|\mathcal{D}) \to 0$. From \pref{eq:ei}, it can be further obtained that $\mathrm{DPOF}(\tilde{\mathbf{x}}) \sigma_f(\tilde{\mathbf{x}}) \to 0$. This completes the proof for the first statement.

    The second statement is proven by virtue of the global search ability of EI and corresponding dense evaluated solutions in $\chi$ if the NEB property is met. We complete the proof by providing the following facts: $i)$ $\mathrm{DPOF}(\tilde{\mathbf{x}}) \sigma_f(\tilde{\mathbf{x}}) \to 0$ holds from the first statement; $ii)$ $\mathrm{DPOF}(\tilde{\mathbf{x}}) \ge \mathrm{POF}(\tilde{\mathbf{x}})$ since $\rho^i\ge 0$; $iii)$ the variance of $\mathrm{POF}(\tilde{\mathbf{x}})$ will not go to zero according to its definition \cite{BachocHP20}; and $iv)$ $\sigma_f(z \vert \mathcal{F}_N) \to 0$ for all sequences accordingly. Based on these facts, the sequence is almost surely dense in $\chi$ if the GP models have the NEB property according to \pref{def:neb}. As a result, $f^*_\mathcal{D}$ from any sequence converges to $f^*_\chi$ almost surely when $N \to \infty$.
\end{proof}

\fancyhead[L]{
    \emph{Supplementary document: Technical Appendices}
}

\section*{Section B. Empirical analysis of EICB}
\label{sec:empirical}
\subsection*{B.1 The illustrative example in  \pref{fig:illustrationEICB}}
\label{sec:empirical_eicb}
The synthetic problem considered in  \pref{fig:illustrationEICB} is given as:
\begin{equation}
    \begin{aligned}
        \mathrm{minimize} \quad & \cos(5x) - \sin(x)\sin(2x), \\
        \mathrm{subject\ to} \quad & \cos(5x) - \sin(x)\sin(2x)\leq 0,
    \end{aligned} \tag{B.1}\label{eq:ill_example}
\end{equation}
where the objective function and constraint function share the same analytical format. The total search space is $\Omega=[0, 10]$, while in the left of  \pref{fig:illustrationEICB} only a sub-region ($[2.5, 5.0]$) is plotted for brevity. The initial evaluated points in the left of \pref{fig:illustrationEICB} include: $2$ feasible points at $[4.25, 4.7]$, and $4$ infeasible points $[2.5, 3.7, 4.2, 5.0]$. We further assume that both objective and constraint are fully observable to efficiently reveal the difference between EIC and EICB, while in all other examples and experiments, we consider POCOPs. The GPR models with Mat\'ern $5/2$ kernel and constant mean function are used to build the surrogate models.

In the right of  \pref{fig:illustrationEICB}, we conduct $5$ repeated experiments in which $10$ solutions are uniformly sampled for initialization. We use consistent random seeds across different acquisition functions to obtain the same initialization. The total budget for optimization is $8$. The median, $1/4$, and $3/4$ quantiles of the best observed objective values are plotted.

\subsection*{B.2 Another instance of the exploration function}
\label{sec:empirical_emub}
The presented design of an exploration function in  \pref{eq:pob} and  \pref{eq:exploration_function} is relatively conservative compared to the level-set estimation and active learning techniques since it only considers boundaries with high probability. We are going to introduce another exploration function that aims for uncertainty reduction of the global feasible regions. Inspired by \cite{BichonESMM08}, this can be achieved by modifying the utility function in  \pref{eq:pob} into
\begin{align}
    \mathrm{MUB}^i(\tilde{\mathbf{x}})= & \max\{\varepsilon(\tilde{\mathbf{x}}) - \vert \tilde{g}_i(\tilde{\mathbf{x}})  \vert, 0 \} \nonumber\\
    = & \left\{ 
        \begin{array}{cc}
            \varepsilon(\tilde{\mathbf{x}}) -  \tilde{g}_i(\tilde{\mathbf{x}}) , & \tilde{g}_i(\tilde{\mathbf{x}}) \in \left[0, \; \varepsilon(\tilde{\mathbf{x}}) \right], \\
            \varepsilon(\tilde{\mathbf{x}}) +  \tilde{g}_i(\tilde{\mathbf{x}}), & \tilde{g}_i(\tilde{\mathbf{x}}) \in \left[-\varepsilon(\tilde{\mathbf{x}}), \; 0 \right), \\
            0, & \mathrm{otherwise}.
        \end{array}
    \right.
    \tag{B.2}
    \label{eq:mub}
\end{align}
We name it the most uncertain boundary (MUB) to indicate that it facilitates a reduction of the uncertainty of global feasible regions more efficiently.
\begin{lemma}
    Given the utility function in \pref{eq:mub}, the expectation of MUB (EMUB) over the predicted distribution of $\tilde g_i(\tilde{\mathbf{x}})$ takes the following form as
    \begin{align}
        \mathrm{EMUB}^i(\tilde{\mathbf{x}})=  & \varepsilon(\tilde{\mathbf{x}}) \Big( \Phi\left(g^+_i(\tilde{\mathbf{x}})\right) -  \Phi\left(g^-_i(\tilde{\mathbf{x}})\right) \Big) \nonumber \\
        & - \sigma_g^i(\tilde{\mathbf{x}}) \Big( 2\phi\left(- \bar g_i(\tilde{\mathbf{x}})\right) \nonumber\\
        & \quad \quad \quad \quad  - \phi\left(g^+_i(\tilde{\mathbf{x}})\right) -\phi\left(g^-_i(\tilde{\mathbf{x}})\right) \Big) \nonumber \\
        & + \mu_g^i(\tilde{\mathbf{x}}) \Big( 2\Phi\left(-\bar g_i(\tilde{\mathbf{x}})\right) \nonumber \\
        & \quad \quad \quad \quad  -\Phi\left(g^+_i(\tilde{\mathbf{x}})\right) - \Phi\left(g^-_i(\tilde{\mathbf{x}})\right) \Big),
        \tag{B.3}
        \label{eq:new_exploration_function}
    \end{align}
where $\bar g_i(\tilde{\mathbf{x}}) = \mu_g^i(\tilde{\mathbf{x}}) / \sigma_g^i(\tilde{\mathbf{x}})$, $g^+_i(\tilde{\mathbf{x}}) = \beta - \mu_g^i(\tilde{\mathbf{x}}) / \sigma_g^i(\tilde{\mathbf{x}})$, $ g_i^-(\tilde{\mathbf{x}}) = -\beta - \mu_g^i(\tilde{\mathbf{x}}) / \sigma_g^i(\tilde{\mathbf{x}})$, $\Phi$ and $\phi$ denote the cumulative distribution function and probability density function of $\mathcal{N}(0, 1)$, respectively. 

\label{lemma:emub}
\end{lemma}

\begin{proof}
Let $\tilde{g}_i(\tilde{\mathbf{x}}) = \mu_g^i(\tilde{\mathbf{x}}) + \sigma_g^i(\tilde{\mathbf{x}}) \epsilon$ with $\epsilon \sim \mathcal{N}(0, 1)$. The expectation of the utility function of MUB is derived by

\begin{align}
    & \mathbb{E}_{\tilde{g}_i(\tilde{\mathbf{x}}) \sim \mathcal{N}(\mu_g^i(\tilde{\mathbf{x}}), \sigma_g^i(\tilde{\mathbf{x}}))} \left[ \mathrm{MUB}^i(\tilde{\mathbf{x}}) \right] \nonumber \\
    = & \mathbb{E}_{\epsilon \sim \mathcal{N}(0, 1)} \left[ \mathrm{MUB}^i(\tilde{\mathbf{x}}) \right] \nonumber\\
    =& \int_{\varepsilon \ge \vert  \mu_g^i(\tilde{\mathbf{x}}) + \sigma_g^i(\tilde{\mathbf{x}}) \epsilon \vert} \left( \varepsilon - \vert  \mu_g^i(\tilde{\mathbf{x}}) + \sigma_g^i(\tilde{\mathbf{x}}) \epsilon \vert \right) \phi\left(\epsilon\right) d\epsilon\nonumber\\
    = & \int_{\bar g_i(\tilde{\mathbf{x}})}^{ g^+_i(\tilde{\mathbf{x}})} \left( \varepsilon -   \mu_g^i(\tilde{\mathbf{x}}) - \sigma_g^i(\tilde{\mathbf{x}}) \epsilon  \right) \phi\left(\epsilon\right) d\epsilon \nonumber \\ 
    & + \int_{ g_i^-(\tilde{\mathbf{x}})}^{\bar  g_i(\tilde{\mathbf{x}})} \left( \varepsilon +   \mu_g^i(\tilde{\mathbf{x}}) + \sigma_g^i(\tilde{\mathbf{x}}) \epsilon  \right) \phi\left(\epsilon\right) d\epsilon \nonumber \\
    = & \varepsilon \left( \Phi\left(g_i^+(\tilde{\mathbf{x}})\right) -  \Phi\left(g_i^-(\tilde{\mathbf{x}})\right) \right) \nonumber\\
    & + \mu_g^i(\tilde{\mathbf{x}}) \left( 2\Phi(\bar g_i(\tilde{\mathbf{x}})) -\Phi(g_i^+(\tilde{\mathbf{x}})) - \Phi(g_i^-(\tilde{\mathbf{x}})) \right) \nonumber \\
    & + \sigma_g^i(\tilde{\mathbf{x}}) \left(\int_{g_i^-(\tilde{\mathbf{x}})}^{\bar g_i(\tilde{\mathbf{x}})} \epsilon \phi\left(\epsilon\right) d\epsilon - \int_{\bar g_i^+(\tilde{\mathbf{x}})}^{g_i^-(\tilde{\mathbf{x}})} \epsilon \phi\left( \epsilon \right) d\epsilon  \right). \tag{B.4}
\end{align}
Note the following result holds:
\begin{align}
    \int_{a}^{b} x \phi(x) dx = \phi(a) - \phi(b). \tag{B.5}\label{eq:epob_lemma}
\end{align}
Therefore, the \pref{eq:new_exploration_function} holds.
\end{proof}

Based on \pref{lemma:emub}, we propose a new exploration function as follows:
\begin{equation}
    \rho^i(\tilde{\mathbf{x}}, \gamma^i) = \frac{\mathrm{EMUB}^i(\tilde{\mathbf{x}}) }{\gamma^i},
    \tag{B.6}
    \label{eq:exploration_function_emub}
\end{equation}
where $\gamma^i>0$ denotes a scale parameter. 

\begin{figure*}[thb]
\centering
\includegraphics[width=\linewidth]{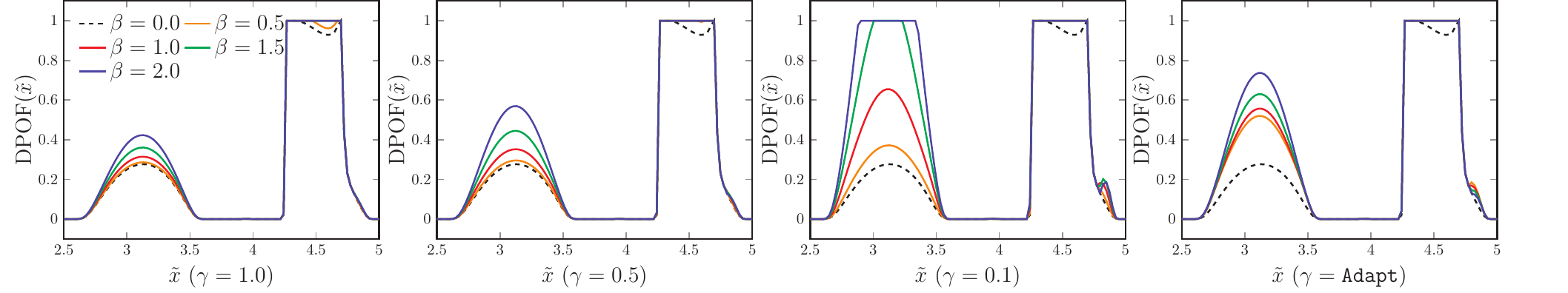}
\figuretag{B-1}
\caption{Illustration of different exploration ability of DPOF with \pref{eq:new_exploration_function} and different settings of $\beta$ and $\gamma$.}
\label{fig:app_eicb}
\end{figure*}


As depicted in \pref{fig:app_eicb}, both $\gamma^i$ and $\beta$ can effect the exploration ability of \pref{eq:new_exploration_function}. To mitigate the complexity of configuration, we introduce an adaptation law of $\gamma^i$ as
\begin{equation}
    \gamma^i = \texttt{Adapt}(\tilde{\mathbf{x}}) = \max_{\tilde{\mathbf{x}}} \mathrm{EMUB}^i(\tilde{\mathbf{x}}) \Phi_g^i(\lambda).\tag{B.7}
\end{equation}
This law enables $\gamma$ to scale $\mathrm{EMUB}^i(\tilde{\mathbf{x}})$ within $[0, 1]$. As shown in \pref{fig:app_eicb}, \texttt{Adapt} leads to a more reasonable strategy of weight allocation. 
\begin{remark}
The dissections to each term in \pref{eq:new_exploration_function} were made in \cite{RanjanBM08} where the utility function was slightly different from  \pref{eq:mub} (they used a squared form as $\varepsilon^2(\tilde{\mathbf{x}}) - \tilde{g}_i^2(\tilde{\mathbf{x}})$ However, the derivation made by Ranjan et al. \cite{RanjanBM08} was incorrect \cite{BectGLPV12}). In general, the first and third terms in \pref{eq:new_exploration_function} suggest candidates on the most uncertainty boundaries, while the second term in \pref{eq:new_exploration_function} suggests candidates from the interior of regions that remain uncertain due to limited observations.
\end{remark}
\begin{remark}
    The alternative exploration function performs more aggressively because it explicitly contains the standard deviation $\sigma_g^i$ that facilitates a reduction of global uncertainty of feasible regions. Unlike the POB and  \pref{eq:exploration_function}, EMUB is no longer a conservative design and may assign a high weight to a candidate solution that has low POF.
\end{remark}

\subsection*{B.3 Comparative study of EI-based CBO methods}
Since this work does not aim at calibrating the involved hyperparameters, such as $\gamma^i$ and $\beta$, for better performance in individual problems, we compare three EI-based CBO methods with fixed parameter settings: $1)$ \texttt{EIC} without hyperparameters, $ii)$ \texttt{EICB-POB} using  \pref{eq:exploration_function} with $\beta=1.96$ (i.e. the $95\%$ confidence level), and $iii)$ \texttt{EICB-MUB} using  \pref{eq:exploration_function_emub} with $\beta=1.96$ and $\gamma^i$ determined by \texttt{Adapt}. Note that for \texttt{EICB-EMUB}, $\beta$ and $\gamma^i$ are coupled. Despite with \texttt{Adapt}, different values of $\beta$ can still result in distinct optimization trajectories. This is one reason for our recommendation for  \pref{eq:exploration_function} as the exploration function, without paying much effort on parameter selections.

We conduct four experiments in S1: $5$D KBF, $10$D KBF, XGB-H, and Lunar, as presented in Section \ref{sec:problems} of the benchmark suite. The results are shown in  \pref{fig:app_exp}. It is observed that, for KBF, \texttt{EICB-MUB} does not outperform \texttt{EIC}. Differently, for XGB-H and Lunar, \texttt{EICB-MUB} is competitive with \texttt{EICB-POB}. Furthermore, it is noticed that the value deviation of \texttt{EICB-MUB} is larger than \texttt{EICB-POB} during optimization processes, which suggests that more exploration is introduced in \texttt{EICB-MUB}. However, this aggressive design can lead to inefficiency for some problems such as KBF and XGB-H, which is another reason for our preference of using  \pref{eq:exploration_function}. Note that we can expect improvements of \texttt{EICB-MUB} by appropriately tuning $\beta$ and $\gamma^i$, which is however out of the scope of this work.


\begin{figure*}[thb]
\centering
\includegraphics[width=\linewidth]{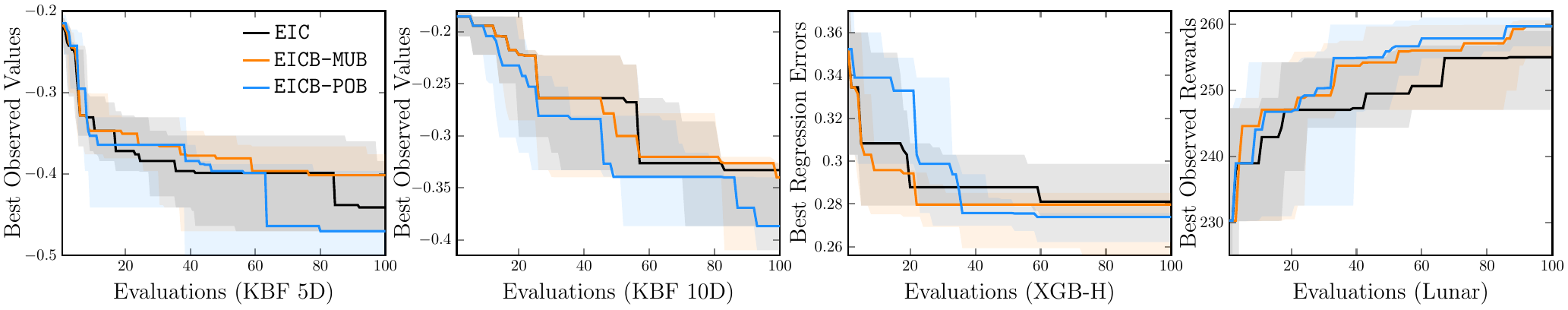}
\figuretag{B-2}
\caption{Optimization trajectories of different tasks using EI-based methods in S1.}
\label{fig:app_exp}
\end{figure*}

\section*{Section C. Expectation Propagation for HLGP}
\label{sec:ep_app}

The posterior of the HLGP can be calculated by:
\begin{equation}
    p(\tilde{\mathbf{g}}_i \vert X, \mathbf{g}_i) =
    \frac{1}{Z} p(\tilde {\mathbf{g}}_i\vert X) 
    \prod_{k=1}^N p( {g}^k_i\vert \tilde{g}_i^k),
    \label{eq:posteriorGP_app}
    \tag{C.1}
\end{equation}
where $Z$ is the normalization factor. For the $k$-th observation $g_i^k$, EP assigns it an un-normalized Gaussian distribution $t_i^k \triangleq \tilde{Z}_i^k\mathcal{N}\left( \tilde\mu_i^k, \tilde\sigma_i^{k\;2} \right)$ to locally approximate its exact likelihood. In this vein, the posterior is approximated by:
\begin{gather}
    p(\tilde{\mathbf{g}}_i \vert X, \mathbf{g}_i) \approx  \frac{1}{Z_{\mathrm{EP}}} p(\tilde{\mathbf{g}}_i\vert X) \prod_{k=1}^N t_i^k = \mathcal{N}(\bm{\mu}_g^i, \Sigma_g^i) \nonumber \\
    \text{with }~\bm{\mu}_g^i= \Sigma_g^i {\tilde{\Sigma}_{g}}^{i\;-1} \tilde{\bm{\mu}}_g^i ~\text{ and }~ \Sigma_g^i = (K + {\tilde{\Sigma}_{g}}^{i\;-1})^{-1},
    \label{eq:posteriorEP_app}
    \tag{C.2}
\end{gather}
where $\tilde{\bm{\mu}}_g^i = \left(\tilde\mu_i^1,\cdots,\tilde\mu_i^N\right)^\top$, $\tilde\Sigma_g^i$ denotes a diagonal matrix with the $k$-th element $\tilde{\sigma}_i^{k\;2}$, and $Z_{\mathrm{EP}}$ is the marginal likelihood. For HLGP, the site parameters $\tilde{Z}_i^k$, $\tilde\mu_i^k$ and $\tilde\sigma_i^k$ in $t_i^k$ are determined by the following laws.
\begin{itemize}
    \item \emph{Law 1.} The site parameters of an exact Gaussian likelihood $\mathcal{N}(g_i(\mathbf{x}^k), \sigma^2)$ are directly assigned by 
    \begin{equation}
        \tilde{Z}_i^k = 1,\quad \tilde{\mu}_i^k = g_i(\mathbf{x}^k),\quad \tilde{\sigma}_i^k = \sigma. \tag{C.3}
    \end{equation}
    \item \emph{Law 2.} The site parameters of a non-Gaussian likelihood are computed by the moment match \cite{RiihimakiV10}. First, the marginal for $\tilde{g}_i^k$ from \eqref{eq:posteriorEP_app} is $\mathcal{N}(\bar \mu^i_k, \bar \sigma^{i\;2}_{k})$, where $\bar \mu^i_k$ and $\bar \sigma^{i\;2}_{k}$ denote the $i$-th element of $\bm{\mu}_g^i$ and $i$-th diagonal element of $\Sigma_g^i$. Then, the cavity parameters, $\bar \mu^i_{-k}$ and $\bar \sigma^{i\;2}_{-k}$, can be computed by
    \begin{equation}
    \begin{split}
        \bar \mu^i_{-k} = &\bar \sigma^{i\;2}_{-k} \left( {\bar \sigma^{i\;-2}_{k}} \bar \mu_k^i - \tilde{\mu}_i^k \tilde{\sigma}_{i}^{k\;-2}\right), \\
        \bar \sigma^{i\;2}_{-k} =& \left(\bar \sigma^{i\;-2}_{k} - \tilde{\sigma}_i^{k\; -2} \right)^{-1}.
    \end{split}
    \tag{C.4}
    \end{equation}
    The desired moments on the true likelihood are:
    \begin{gather}
         \hat \mu_i^k =  \bar \mu^i_{-k} + \frac{\bar \sigma^{i\;2}_{-k} \phi(z_i^k)}{\Phi(z_i^k)\sqrt{\alpha^2 + \bar \sigma^{i\;2}_{-k}}}, \nonumber\\
        \hat \sigma_i^{k\;2} =  \sigma^{k\;2}_{-i} - \frac{\bar \sigma^{i\;4}_{-k}\phi({z_i^k})}{\left(\alpha^2 + \bar \sigma^{i\;2}_{-k}\right) \Phi(z_i^k)} \left(z_i^k + \frac{\phi({z_i^k})}{\Phi({z_i^k})} \right), \nonumber \\
        \hat{Z}_i^k =  \Phi(z_i^k),
    \tag{C.5}
    \end{gather}
    where $z_i^k= \frac{\bar \mu^i_{-k}}{\sqrt{\alpha^2 +\bar \sigma^{i\;2}_{-k}}}$.
    Henceforth, the site parameters can be computed by matching the above moments. Mathematically, it means
    \begin{equation}
    \begin{split}
        \tilde \mu_i^k = & \tilde\sigma_i^{k\;2} \left(\hat \sigma_i^{k\;-2}\hat\mu_i^k - \sigma_{-i}^{k\;-2}\mu^k_{-i}\right),\\
        \tilde\sigma_i^{k\;2} = & \left(\hat\sigma_i^{k\;2} - \sigma_{-i}^{k\;-2}\right)^{-1}.
    \end{split}
    \tag{C.6}
    \end{equation}
    As the un-normalized term $\tilde{Z}_i^k$ does not affect the modeling result, we omit its computation here. Detailed derivation of above processes is referred to \cite{GPML,RiihimakiV10}.

    \item \emph{Law 3.} Repeatedly compute the above two laws for $k = 1,\dots,N$, until all values of the site parameters converge. 
\end{itemize}

\begin{remark}
Computational instability and invalid operation may occur when updating the above parameters. We suggest several ways to alleviate this phenomenon. The first way is using calculation tricks with stronger numerical stability, refer to \cite{GPML} for practical implementation of an EP algorithm. When non-singularity exists, one can fix the covariance matrices by manipulating the eigenvalues, see the implementation of \cite{MarcoBKHRT21}. Moreover, alternative choices of the Non-Gaussian likelihoods, e.g., step function studied in \cite{garnett_bayesoptbook_2023}, can be more effective for a description of the truncated distribution.
\end{remark}

\section*{Section D. Experiment Settings}
\label{sec:experiment}

\subsection*{D.1 Algorithm Implementation}
\subsubsection*{CBO methods} 
In our project, we implement all the algorithms on two widely used platforms, namely Gpflow \footnote{\href{https://gpflow.github.io/GPflow}{https://gpflow.github.io/GPflow}} and GPy \footnote{\href{https://gpy.readthedocs.io}{https://gpy.readthedocs.io}}. The \texttt{EIC} algorithm has been officially presented in Trieste \cite{MatthewsWNFBLGH17}. The implementations of \texttt{TSC}\footnote{\href{https://github.com/pytorch/botorch}{https://github.com/pytorch/botorch}} \cite{ErikssonP21} and \texttt{MESC}\footnote{\href{https://github.com/takeuchi-lab/CMES-IBO}{https://github.com/takeuchi-lab/CMES-IBO}} \cite{TakenoTSK22} are planted from their official open-source projects. For \texttt{MESC}, we choose to use $20$ samples with a $1\,000$ grid in Thompson sampling. For \texttt{TSC}, we fix the grid number as $1\,000$. In our experiments, when increasing the number of samples and grids, we did not observe a significant improvement in the optimization efficiency. For \texttt{CBOB}, we fix $\beta=1.96$ for $95\%$ confidence level. Note that tuning $\beta$ will not lead to significantly different behaviors of \texttt{CBOB} with the conservative exploration function designed in \pref{eq:exploration_function}.

\subsubsection*{Gaussian processes} 
Uniformly, the GPR models are used to build the surrogate models of objective functions, where observations are sequentially obtained with a mask on infeasible solutions. We use GPR, GPC, and HLGP models to build models of constraints for specific algorithms and tasks. All GP models take a constant mean function and the Mat\'ern $5/2$ kernel. The GPC models are created by the default builder of Gpflow/GPy with initial observations \cite{MatthewsWNFBLGH17} and optimized by the variational inference or EP. For HLGPs, we first modify the raw observations using the EP algorithm in consideration of different likelihood functions. Then the GP model with heterogeneous noises is utilized.

\subsubsection*{More configurations} 
All acquisition functions are optimized by the L-BFGS-B method with $1\,000$ iterations. The hyperparameters in all GP models are optimized according to the batch optimizer embedded in GPflow and GPy. All illustrative examples and experiments are performed on a desktop with Intel(R) Xeon(R) CPU E5-2620 v4 (2.10GHz) and NVIDIA GeForce GTX 1080Ti GPU.

\subsection*{D.2 Synthetic Benchmark}
This section delineates the configurations of individual synthetic benchmark problems. For brevity, we present a comprehensive table (\pref{tab:synthetic_bench}) including the basic settings of each problem regarding black-box optimization.

\begin{table}[t]
    \centering
    \begin{tabular}{|c|c|c|c|c|}
    \hline
       Problem  &  Dimension  & Objective & constraint & search space \\
    \hline
       \multirow{2}*{KBF \cite{Keane94}}  &  \multirow{2}*{$n=10$, $m=2$} &  \multirow{2}*{$- \vert \frac{\sum_{i=1}^{10}\cos^4(x_i) - 2\Pi_{i=1}^{10}\cos^2(x_i)}{\sqrt{\sum_{i=1}^{10} i x_i^2}} \vert$} & $ 0.75 - \Pi_{i=1}^{10} x_i < 0$ & \multirow{2}*{$[0, 10]^{10}$} \\
       ~  &  ~ & ~ & $\sum_{i=1}^{10} x_i - 75 < 0$ & ~ \\
       \hline
       Ackley \cite{ErikssonP21}  &  $n=10$, $m=1$ & $a=20$, $b=0.2$, $c=2\pi$ & $\sum_{i=1}^{10} x_i< 0$ & $ [-5, 5]^{10}$ \\
       \hline
       \multirow{2}*{WBD \cite{Deb00}}  &  \multirow{2}*{$n=4$, $m=5$} &  \multicolumn{2}{|c|}{\multirow{2}*{please refer to \cite{Deb00}}}  & $x_1,x_4 \in [0.125, 5]$ \\
       ~ &  ~ &  \multicolumn{2}{|c|}{~} & $x_2,x_3 \in [0.1, 10]$ \\
       \hline
       \multirow{3}*{PVD \cite{CoelloM02}}  &  \multirow{3}*{$n=4$, $m=4$} &  \multicolumn{2}{|c|}{\multirow{3}*{please refer to \cite{CoelloM02}}}  & $x_1,x_2 \in [0, 20]$ \\
       ~ &  ~ &  \multicolumn{2}{|c|}{~} & $x_3 \in [10, 50]$ \\
       ~ &  ~ &  \multicolumn{2}{|c|}{~} & $x_4 \in [150, 200]$ \\
    \hline
    \end{tabular}
    \caption{Configurations of synthetic benchmark problems.}
    \label{tab:synthetic_bench}
\end{table}

\subsubsection*{KBF}
The Keane bump function (KBF) is a synthetic constrained optimization problem for testing constrained optimization methods \cite{Keane94}. It consists of one objective function and two constraint functions. We consider the $10$D KBF with the search space as $[0, 10]^{10}$.

\subsubsection*{Ackley}
This problem is studied in \cite{ErikssonP21}. The Ackley function is used with the recommended variable values\footnote{\href{https://www.sfu.ca/~ssurjano/ackley.html}{https://www.sfu.ca/$\sim$ssurjano/ackley.html}}. The single constraint considered in this problem is $\sum_{i=1}^{10} x_i \leq 0$. Besides, the search space is given by $\left[-5.0, 5.0 \right]^{10}$, and the optimum is located at the origin of the axes such that $f(\vec{0}) = 0$.

\subsubsection*{WBD}
The standard definition of the welded beam design problem is given in \cite{Deb00}. We use the search space $x_1, x_4 \in [0.125, 5]$ and $x_2, x_3 \in [0.1, 10]$, which is also used in \cite{CoelloM02}. \texttt{EIC} finds a good solution within $20$ FEs, but will be stagnated in the following evaluations. Differently, \our\ has less efficiency in the early stage, but can gradually find a good solution without prolonged stagnation.

\subsubsection*{PVD}
The standard definition of the pressure vessel design problem can be found in \cite{CoelloM02}. The search space is adopted from \cite{ErikssonP21} as $x_1 \in [0, 20 ]$, $x_2 \in [0, 20 ]$, $x_3 \in [10, 50]$ and $x_4 \in [150, 200]$. To deal with the discrete variables $x_1$ and $x_2$, we uniformly round them to be integers. In this experiment, \texttt{MESC}$_c$ fails to find a better solution within $100$ function evaluations, therefore omitted from the final results.

\subsection*{D.3 Real-World Benchmark}
We consider two kinds of real-world problems, i.e., the hyperparameter optimization (HPO) for a good machine learning model and the reinforcement learning for a better controller. Therein, we consider the model size and energy consumption as the analytically unknown constraints. The constraint thresholds are determined by the mean values of the model size or total energy after $2\,000$ repetitive experiments with randomly sampled solutions. Therefore, about $50\%$ solutions of the initial evaluations are infeasible. A brief table (\pref{tab:real_world_bench}) is given with basic configurations of each problem.

\begin{table}[t]
    \centering
    \begin{tabular}{|c|c|c|c|c|c|}
    \hline
       Problem  &  Dimension  & Objective & constraint & threshold & search space\\
    \hline
    \multirow{7}*{XGB-H}  &  ~  & \multirow{7}*{MSE} & \multirow{7}*{model size} & \multirow{7}*{$50\,000$} & learning rate: $[2^{-10}, 1]$ (log)\\
    ~ &  ~  & ~ & ~ & ~ & maximum depth: $[1, 15]$ (Int) \\
    ~ &  $n=7$  & ~ & ~ & ~ & subsample ratio: $[0.01, 1]$ \\
    ~ &  ~  & ~ & ~ & ~ & L2 regularization: $[2^{-10}, 2^{10}]$ (log)\\
    ~ &  $m=1$  & ~ & ~ & ~ & L1 regularization: $[2^{-10}, 2^{10}]$ (log)\\
    ~ &  ~  & ~ & ~ & ~ & minimum weight sum: $[1, 2^7]$ (log)\\
    ~ &  ~  & ~ & ~ & ~ & estimator number: $[1, 2^8]$ (log, Int)\\
    \hline
    \multirow{8}*{MLP-D}  &  ~  & ~ & \multirow{8}*{model size} & \multirow{8}*{$107\,000$} & learning rate: $[10^{-5}, 1]$ (log)\\
        ~ &  ~  & ~ & ~ & ~ & hidden layer 1: $[2^2, 2^8]$ (log, Int) \\
    ~ &  ~ & fraction of  & ~ & ~ & hidden layer 2: $[2^2, 2^8]$ (log, Int) \\
    ~ &  $n=8$ & correctly  & ~ & ~ & batch size: $[2^2, 2^8]$ (log, Int)\\
    ~ &  $m=1$ & classified  & ~ & ~ & L2 regularization: $[10^{-8}, 10^{-3}]$ (log)\\
    ~ &  ~ & samples & ~ & ~ & Adam decay rate 1: $[0, 0.9999]$ (log)\\
    ~ &  ~  & ~ & ~ & ~ & Adam decay rate 2: $[0, 0.9999]$ (log)\\
    ~ &  ~  & ~ & ~ & ~ & tolerance: $[10^{-6}, 10^{-2}]$ (log)\\
    \hline
    Lunar &  $\begin{array}{c} n=12 \\ m=1 \end{array}$  & reward & fuel cost & $40.0$ & tolerance: $[0, 2]^{12}$\\
    \hline
    Swim. & $\begin{array}{c} n=16 \\ m=1 \end{array}$ & reward & $\begin{array}{c} \text{energy} \\ \text{cost} \end{array}$ & 1.2 & $[-1, 1]^16$\\
    \hline
    \end{tabular}
    \caption{Configurations of synthetic benchmark problems.}
    \label{tab:real_world_bench}
\end{table}

\subsubsection*{XGB-H}
The California housing is a regression problem with $20\,640$ samples. The dataset is prepared with a $0.25$ train/test split in scikit-learn. The XGBoost regressor\footnote{\href{https://github.com/dmlc/xgboost/}{https://github.com/dmlc/xgboost/}} is used to fit the dataset, for which we optimize $7$ configurable parameters with different algorithms. The parameters include: learning rate $[2^{-10}, 1]$ (log), maximum depth $[1, 15]$ (Int), subsample ratio of columns $[0.01, 1]$, L2 regularization term $[2^{-10}, 2^{10}]$ (log), L1 regularization term $[2^{-10}, 2^{10}]$ (log), minimum sum of instance weight in a child $[1, 2^{7}]$ (log), and number of estimators $[1, 2^8]$ (log, Int). We round the solutions for integer parameters. With $2\,000$ repetitive experiments, the model size threshold is $50\,000\mathrm{bytes}$.

\subsubsection*{MLP-D}
The digits dataset contains $1\,797$ images of handwriting digits from $1$-$10$. It is a common classification problem and prepared with a $0.25$ train/test split in scikit-learn. We use a Multi-layer Perceptron (MLP) classifier to fit the dataset, for which we optimize $8$ configurable parameters with different algorithms. The parameters include: initial learning rate $[10^{-5}, 1]$ (log), two sizes of the hidden layers $[2^2, 2^8]^2$ (log, Int), batch size $[2^2, 2^8]$ (log, Int), L2 regularization term $[10^{-8}, 10^{-3}]$ (log), two exponential decay rates in Adam $[0, 0.9999]^2$, tolerance for the optimization $[10^{-6}, 10^{-2}]$ (log),. With $2\,000$ repetitive experiments, the model size threshold is $107\,000\mathrm{bytes}$.

\subsubsection*{Lunar}
In this experiment, we should design a controller that renders a rocket to land on the target position with more reward determined by $i)$ the distance to the target, $ii)$ fuel consumption, and $iii)$ whether crashed. As studied in \cite{ErikssonPGTP19}, a heuristic controller with $12$ parameters can be designed to control the rocket. We use 'LunarLander-v2' in simulations, where the fuel consumption is separated as an energy constraint. In this task, the fuel consumption should be less than $40.0$. The search space is $[0, 2]^{12}$. To create a deterministic environment, we fix the terrain and initial state of the rocket.

\subsubsection*{Swimmer}
In this experiment, the goal is to control a swimmer to move as fast
as possible towards one direction. In the MujoCo environment \cite{TodorovET12}, we use 'swimmer-v4' which contains $8$D output state space and $2$D action space. As designed in \cite{WangFT20}, the linear feedback controller is utilized, resulting in $16$ parameters in the control gain matrix to be determined. Similarly, we separate the control cost, which was considered in the total reward, as an energy constraint. The constraint threshold is set as $1.2$. The search space is set as $[-1, 1]^{16}$. To obtain a robust controller against small noises, in each evaluation, the simulation runs $5$ times with different initialization seeds.

\section*{E. Further Discussions}
\label{sec:further}

\subsection*{E.1 Limitations}
Our investigations of EICB have two limitations. First, we only study the asymptotic convergence of EICB, from which we cannot give definite answers on whether EICB outperforms EIC at the convergence rate in theory. On the other hand, the empirical design of exploration functions needs to be further studied for better guidance to the practitioners. A promising improvement is a further balancing design considering risks and robustness, regarding infeasible evaluations or aggressive exploration. For safety-critical scenarios, EICB should be capable of positing the quality of candidate solutions.

In addition to EICB, the design of HLGP models can be further improved in consideration of the uncertainty on observed solutions. Specifically, the HLGP in its current form cannot eliminate the uncertainty of an evaluated infeasible point, as depicted in  \pref{fig:HLGP_EP}. Note that GPC also fails to eliminate the uncertainty. Differently, a classifier such as SVM or MLP with appropriate model capacity can easily achieve this goal. We believe a more complex design of the surrogate models can be leveraged to better exploit the mixed observations.

\subsection*{E.2 Potential Impact}
The proposed DPOF is a new constraint-handling technique that can be easily integrated with non-negative acquisition functions, such as EI, probability of improvement, parzen estimator, and density-ratio estimation. With other acquisition functions, the theoretical analysis on both convergence and regret bound may be conducted, potentially leading to a more principled design of DPOF. Besides, our work gives criteria for the research focusing on the adaptive design of CBO methods by leveraging information from similar tasks for better optimization with unknown constraints.

%% file: tech_main.bbl
\begin{thebibliography}{10}
\providecommand{\url}[1]{#1}
\csname url@samestyle\endcsname
\providecommand{\newblock}{\relax}
\providecommand{\bibinfo}[2]{#2}
\providecommand{\BIBentrySTDinterwordspacing}{\spaceskip=0pt\relax}
\providecommand{\BIBentryALTinterwordstretchfactor}{4}
\providecommand{\BIBentryALTinterwordspacing}{\spaceskip=\fontdimen2\font plus
\BIBentryALTinterwordstretchfactor\fontdimen3\font minus
  \fontdimen4\font\relax}
\providecommand{\BIBforeignlanguage}[2]{{%
\expandafter\ifx\csname l@#1\endcsname\relax
\typeout{** WARNING: IEEEtran.bst: No hyphenation pattern has been}%
\typeout{** loaded for the language `#1'. Using the pattern for}%
\typeout{** the default language instead.}%
\else
\language=\csname l@#1\endcsname
\fi
#2}}
\providecommand{\BIBdecl}{\relax}
\BIBdecl

\bibitem{MarcoBKHRT21}
A.~Marco, D.~Baumann, M.~Khadiv, P.~Hennig, L.~Righetti, and S.~Trimpe, ``Robot
  learning with crash constraints,'' \emph{{IEEE} Robotics Autom. Lett.},
  vol.~6, no.~2, pp. 1439--1446, 2021.

\bibitem{PerroneSJAS19}
\BIBentryALTinterwordspacing
V.~Perrone, I.~Shcherbatyi, R.~Jenatton, C.~Archambeau, and M.~W. Seeger,
  ``Constrained {B}ayesian optimization with max-value entropy search,''
  \emph{CoRR}, vol. abs/1910.07003, 2019. [Online]. Available:
  \url{http://arxiv.org/abs/1910.07003}
\BIBentrySTDinterwordspacing

\bibitem{garnett_bayesoptbook_2023}
R.~Garnett, \emph{{Bayesian Optimization}}.\hskip 1em plus 0.5em minus
  0.4em\relax Cambridge University Press, 2023.

\bibitem{JonesSW98}
D.~R. Jones, M.~Schonlau, and W.~J. Welch, ``Efficient global optimization of
  expensive black-box functions,'' \emph{J. Glob. Optim.}, vol.~13, no.~4, pp.
  455--492, 1998.

\bibitem{BrockmanCPSSTZ16}
G.~Brockman, V.~Cheung, L.~Pettersson, J.~Schneider, J.~Schulman, J.~Tang, and
  W.~Zaremba, ``Openai gym,'' 2016.

\bibitem{SchonlauWJ98}
M.~Schonlau, W.~J. Welch, and D.~R. Jones, \emph{Global Versus Local Search in
  Constrained Optimization of Computer Models}.\hskip 1em plus 0.5em minus
  0.4em\relax Institute of Mathematical Statistics, 1 1998, vol.~34, pp.
  11--25.

\bibitem{GardnerKZWC14}
J.~R. Gardner, M.~J. Kusner, Z.~E. Xu, K.~Q. Weinberger, and J.~P. Cunningham,
  ``Bayesian optimization with inequality constraints,'' in \emph{ICML'14:
  Proc. of the 31th International Conference on Machine Learning},
  vol.~32.\hskip 1em plus 0.5em minus 0.4em\relax JMLR.org, 2014, pp. 937--945.

\bibitem{GelbartSA14}
M.~A. Gelbart, J.~Snoek, and R.~P. Adams, ``Bayesian optimization with unknown
  constraints,'' in \emph{UAI'14: Proc. of the 13th Conference on Uncertainty
  in Artificial Intelligence}.\hskip 1em plus 0.5em minus 0.4em\relax {AUAI}
  Press, 2014, pp. 250--259.

\bibitem{LethamKOB17}
B.~Letham, B.~Karrer, G.~Ottoni, and E.~Bakshy, ``Constrained {B}ayesian
  optimization with noisy experiments,'' \emph{Bayesian Anal.}, vol.~14, no.~2,
  pp. 495 -- 519, 2019.

\bibitem{LamW17}
R.~Lam and K.~Willcox, ``Lookahead {B}ayesian optimization with inequality
  constraints,'' in \emph{NeurIPS'17: Annual Conference on Neural Information
  Processing Systems}, 2017, pp. 1890--1900.

\bibitem{ZhangZF21}
Y.~Zhang, X.~Zhang, and P.~I. Frazier, ``Two-step lookahead {B}ayesian
  optimization with inequality constraints,'' in \emph{NeurIPS'21: Annual
  Conference on Neural Information Processing Systems}, 2021, pp.
  12\,563--12\,575.

\bibitem{LobatoGHAG15a}
J.~M. Hern{\'{a}}ndez{-}Lobato, M.~A. Gelbart, M.~W. Hoffman, R.~P. Adams, and
  Z.~Ghahramani, ``Predictive entropy search for bayesian optimization with
  unknown constraints,'' in \emph{ICML'15: Proc. of the 32nd International
  Conference on Machine Learning 2015}, vol.~37.\hskip 1em plus 0.5em minus
  0.4em\relax JMLR.org, 2015, pp. 1699--1707.

\bibitem{TakenoTSK22}
S.~Takeno, T.~Tamura, K.~Shitara, and M.~Karasuyama, ``Sequential and parallel
  constrained max-value entropy search via information lower bound,'' in
  \emph{ICML'22, Proc. of the 39th International Conference on Machine
  Learning}, vol. 162.\hskip 1em plus 0.5em minus 0.4em\relax {PMLR}, 2022, pp.
  20\,960--20\,986.

\bibitem{BelakariaDD19}
S.~Belakaria, A.~Deshwal, and J.~R. Doppa, ``Max-value entropy search for
  multi-objective {B}ayesian optimization,'' in \emph{NeurIPS'19: Annual
  Conference on Neural Information Processing Systems}, 2019, pp. 7823--7833.

\bibitem{LindbergL15}
D.~V. Lindberg and H.~K. Lee, ``Optimization under constraints by applying an
  asymmetric entropy measure,'' \emph{J. Comput. Graph. Stat.}, vol.~24, no.~2,
  pp. 379--393, 2015.

\bibitem{GramacyGDLRWW16}
R.~B. Gramacy, G.~A. Gray, S.~L. Digabel, H.~K.~H. Lee, P.~Ranjan, G.~N. Wells,
  and S.~M. Wild, ``Modeling an augmented {L}agrangian for black-box
  constrained optimization,'' \emph{Technometrics}, vol.~58, no.~1, pp. 1--11,
  2016.

\bibitem{PichenyGWD16}
V.~Picheny, R.~B. Gramacy, S.~M. Wild, and S.~L. Digabel, ``Bayesian
  optimization under mixed constraints with a slack-variable augmented
  {L}agrangian,'' in \emph{NeurIPS'16: Annual Conference on Neural Information
  Processing Systems}, 2016, pp. 1435--1443.

\bibitem{AriafarCBD19}
S.~Ariafar, J.~Coll{-}Font, D.~H. Brooks, and J.~G. Dy, ``{ADMMBO:} {B}ayesian
  optimization with unknown constraints using {ADMM},'' \emph{J. Mach. Learn.
  Res.}, vol.~20, pp. 123:1--123:26, 2019.

\bibitem{ErikssonP21}
D.~Eriksson and M.~Poloczek, ``Scalable constrained {B}ayesian optimization,''
  in \emph{AISTATS'21: Proc. of the 24th International Conference on Artificial
  Intelligence and Statistics}, vol. 130.\hskip 1em plus 0.5em minus
  0.4em\relax {PMLR}, 2021, pp. 730--738.

\bibitem{BachocHP20}
F.~Bachoc, C.~Helbert, and V.~Picheny, ``Gaussian process optimization with
  failures: Classification and convergence proof,'' \emph{J. Glob. Optim.},
  vol.~78, no.~3, pp. 483--506, 2020.

\bibitem{Candelieri21}
A.~Candelieri, ``Sequential model based optimization of partially defined
  functions under unknown constraints,'' \emph{J. Glob. Optim.}, vol.~79,
  no.~2, pp. 281--303, 2021.

\bibitem{PourmohamadL16}
T.~Pourmohamad and H.~K.~H. Lee, ``Multivariate stochastic process models for
  correlated responses of mixed type,'' \emph{Bayesian Anal.}, vol.~11, pp. 797
  -- 820, 2016.

\bibitem{ZhangDL19}
Y.~Zhang, Z.~Dai, and B.~K.~H. Low, ``Bayesian optimization with binary
  auxiliary information,'' in \emph{UAI'19: Proc. of the 35th Conference on
  Uncertainty in Artificial Intelligence}, vol. 115.\hskip 1em plus 0.5em minus
  0.4em\relax {AUAI} Press, 2019, pp. 1222--1232.

\bibitem{ParrKFH12}
J.~M. Parr, A.~J. Keane, A.~I. Forrester, and C.~M. Holden, ``Infill sampling
  criteria for surrogate-based optimization with constraint handling,''
  \emph{Eng. Optim.}, vol.~44, no.~10, pp. 1147--1166, 2012.

\bibitem{Picheny14}
V.~Picheny, ``A stepwise uncertainty reduction approach to constrained global
  optimization,'' in \emph{AISTATS'14: Proc. of the 17th International
  Conference on Artificial Intelligence and Statistics}, vol.~33.\hskip 1em
  plus 0.5em minus 0.4em\relax JMLR.org, 2014, pp. 787--795.

\bibitem{WangI18}
Z.~Wang and M.~Ierapetritou, ``Constrained optimization of black-box stochastic
  systems using a novel feasibility enhanced kriging-based method,''
  \emph{Comput. Chem. Eng.}, vol. 118, pp. 210--223, 2018.

\bibitem{RanjanBM08}
P.~Ranjan, D.~Bingham, and G.~Michailidis, ``Sequential experiment design for
  contour estimation from complex computer codes,'' \emph{Technometrics},
  vol.~50, no.~4, pp. 527--541, 2008.

\bibitem{BectGLPV12}
J.~Bect, D.~Ginsbourger, L.~Li, V.~Picheny, and E.~V{\'{a}}zquez, ``Sequential
  design of computer experiments for the estimation of a probability of
  failure,'' \emph{Stat. Comput.}, vol.~22, no.~3, pp. 773--793, 2012.

\bibitem{BachocCG21}
F.~Bachoc, T.~Cesari, and S.~Gerchinovitz, ``The sample complexity of level set
  approximation,'' in \emph{AISTATS'21: Proc. of the 24th International
  Conference on Artificial Intelligence and Statistics}, vol. 130.\hskip 1em
  plus 0.5em minus 0.4em\relax {PMLR}, 2021, pp. 424--432.

\bibitem{GPML}
C.~E. Rasmussen and C.~K.~I. Williams, \emph{Gaussian Processes for Machine
  Learning}.\hskip 1em plus 0.5em minus 0.4em\relax The MIT Press, 11 2005.

\bibitem{VazquezB10}
E.~Vazquez and J.~Bect, ``Convergence properties of the expected improvement
  algorithm with fixed mean and covariance functions,'' \emph{J. Stat. Plan.
  Inference}, vol. 140, no.~11, pp. 3088--3095, 2010.

\bibitem{BichonESMM08}
B.~J. Bichon, M.~S. Eldred, L.~P. Swiler, S.~Mahadevan, and J.~M. McFarland,
  ``Efficient global reliability analysis for nonlinear implicit performance
  functions,'' \emph{AIAA J.}, vol.~46, no.~10, pp. 2459--2468, 2008.

\bibitem{RiihimakiV10}
J.~Riihim{\"{a}}ki and A.~Vehtari, ``Gaussian processes with monotonicity
  information,'' in \emph{AISTATS'10: Proc. of the 13th International
  Conference on Artificial Intelligence and Statistics}.\hskip 1em plus 0.5em
  minus 0.4em\relax JMLR.org, 2010, pp. 645--652.

\bibitem{Minka01}
T.~P. Minka, ``Expectation propagation for approximate {B}ayesian inference,''
  in \emph{{UAI} '01: Proc. of the 17th Conference in Uncertainty in Artificial
  Intelligence}.\hskip 1em plus 0.5em minus 0.4em\relax Morgan Kaufmann, 2001,
  pp. 362--369.

\bibitem{PedregosaVGMTGBPWDVPCBPD11}
F.~Pedregosa, G.~Varoquaux, A.~Gramfort, V.~Michel, B.~Thirion, O.~Grisel,
  M.~Blondel, P.~Prettenhofer, R.~Weiss, V.~Dubourg, J.~VanderPlas, A.~Passos,
  D.~Cournapeau, M.~Brucher, M.~Perrot, and E.~Duchesnay, ``Scikit-learn:
  {M}achine learning in {P}ython,'' \emph{J. Mach. Learn. Res.}, vol.~12, pp.
  2825--2830, 2011.

\bibitem{Keane94}
A.~J. Keane, ``Experiences with optimizers in structural design,'' in
  \emph{Proc. of the conference on adaptive computing in engineering design and
  control}, vol.~94, 1994, pp. 14--27.

\bibitem{Deb00}
K.~Deb, ``An efficient constraint handling method for genetic algorithms,''
  \emph{Computer Methods in Applied Mechanics and Engineering}, vol. 186,
  no.~2, pp. 311--338, 2000.

\bibitem{ErikssonPGTP19}
D.~Eriksson, M.~Pearce, J.~R. Gardner, R.~Turner, and M.~Poloczek, ``Scalable
  global optimization via local {B}ayesian optimization,'' in \emph{NeurIPS'19:
  Annual Conference on Neural Information Processing Systems}, 2019, pp.
  5497--5508.

\bibitem{CoelloM02}
C.~A.~C. Coello and E.~Mezura{-}Montes, ``Constraint-handling in genetic
  algorithms through the use of dominance-based tournament selection,''
  \emph{Adv. Eng. Informatics}, vol.~16, no.~3, pp. 193--203, 2002.

\bibitem{WangFT20}
L.~Wang, R.~Fonseca, and Y.~Tian, ``Learning search space partition for
  black-box optimization using monte carlo tree search,'' in \emph{NeurIPS'20:
  Annual Conference on Neural Information Processing Systems}, 2020.

\bibitem{WatanabeH23}
S.~Watanabe and F.~Hutter, ``c-tpe: Tree-structured parzen estimator with
  inequality constraints for expensive hyperparameter optimization,'' in
  \emph{IJCAI'23: Proc. of the 32nd International Joint Conference on
  Artificial Intelligence}, 2023, pp. 4371--4379.

\bibitem{BectBG19}
J.~Bect, F.~Bachoc, and D.~Ginsbourger, ``A supermartingale approach to
  {G}aussian process based sequential design of experiments,''
  \emph{Bernoulli}, vol.~25, no.~4A, pp. 2883--2919, 2019.

\bibitem{GinsbourgerRD16}
D.~Ginsbourger, O.~Roustant, and N.~Durrande, ``On degeneracy and invariances
  of random fields paths with applications in {G}aussian process modelling,''
  \emph{J. Stat. Plan. Inference}, vol. 170, pp. 117--128, 2016.

\bibitem{MatthewsWNFBLGH17}
A.~G. de~G.~Matthews, M.~van~der Wilk, T.~Nickson, K.~Fujii, A.~Boukouvalas,
  P.~Le{\'{o}}n{-}Villagr{\'{a}}, Z.~Ghahramani, and J.~Hensman, ``Gpflow: {A}
  {G}aussian process library using tensorflow,'' \emph{J. Mach. Learn. Res.},
  vol.~18, pp. 40:1--40:6, 2017.

\bibitem{TodorovET12}
E.~Todorov, T.~Erez, and Y.~Tassa, ``Mujoco: {A} physics engine for model-based
  control,'' in \emph{IROS'12: {IEEE/RSJ} International Conference on
  Intelligent Robots and Systems}.\hskip 1em plus 0.5em minus 0.4em\relax
  {IEEE}, 2012, pp. 5026--5033.

\end{thebibliography}
